\documentclass[letterpaper]{article}

\usepackage{natbib}
\usepackage{algorithm}
\usepackage{xspace}
\usepackage{graphicx}
\usepackage{amsmath,amssymb,amsfonts,amsthm,bbm}
\usepackage{graphicx}
\usepackage{cleveref}
\usepackage{xcolor}
\usepackage{balance}
\usepackage{wrapfig}

\usepackage[utf8]{inputenc} 
\usepackage[T1]{fontenc}    
\usepackage{booktabs}       
\usepackage{nicefrac}       
\usepackage{microtype}      

\newtheorem{theorem}{Theorem}
\newtheorem{lemma}{Lemma}
\newtheorem{assumption}{Assumption}

\Crefname{assumption}{Assumption}{Assumptions}

\newcommand{\e}[0]{\mathbb{E}\xspace}
\newcommand{\var}[0]{\mathbb{V}\xspace}
\newcommand{\covar}[0]{\text{Cov}\xspace}
\newcommand{\norm}[1]{\|#1\|\xspace}
\newcommand{\abs}[1]{|#1|\xspace}
\newcommand{\sape}[0]{\text{SAPE}\xspace}
\newcommand{\indicator}[1]{\mathbbm{1}\{#1\}\xspace}
\newcommand{\inner}[2]{\langle#1,#2\rangle\xspace}
\newcommand{\lpspace}[1]{\mathcal L_{#1}\xspace}
\newcommand{\Dataset}{\mathfrak{D}}
\newcommand{\Dset}{\mathcal{D}}
\newcommand{\Sset}{\mathcal{S}}
\newcommand{\Aset}{\mathcal{A}}
\newcommand{\Uset}{\mathcal{U}}
\newcommand{\Fset}{\mathcal{F}}
\newcommand{\Gset}{\mathcal{G}}
\newcommand{\Hset}{\mathcal{H}}
\newcommand{\Nset}{\mathbb{N}}
\newcommand{\pv}{v} 

\usepackage{amsmath,amsfonts,bm}

\def\eqref#1{equation~\ref{#1}}

\def\1{\bm{1}}

\DeclareMathAlphabet{\mathsfit}{\encodingdefault}{\sfdefault}{m}{sl}
\SetMathAlphabet{\mathsfit}{bold}{\encodingdefault}{\sfdefault}{bx}{n}

\DeclareMathOperator*{\argmin}{arg\,min}

\usepackage{times}

\usepackage{fullpage}

\usepackage{paralist}
\def\subfigure{}
\date{}
\title{Off-policy Evaluation in Infinite-Horizon \\Reinforcement Learning with Latent Confounders}

\author{
Andrew Bennett\thanks{Alphabetical order.}
\hspace{10mm}
Nathan Kallus\footnotemark[1] \\
Cornell University\\
\texttt{\{awb222,kallus\}@cornell.edu}
\and
Lihong Li\footnotemark[1]
\hspace{10mm}
Ali Mousavi\footnotemark[1] \\
Google Research\\
\texttt{\{lihong,alimous\}@google.com}
}

\begin{document}
\maketitle

\begin{abstract}
Off-policy evaluation (OPE) in reinforcement learning is an important problem in settings where experimentation is limited, such as education and healthcare. But, in these very same settings, observed actions are often confounded by unobserved variables making OPE even more difficult. We study an OPE problem in an infinite-horizon, ergodic Markov decision process with unobserved confounders, where states and actions can act as proxies for the unobserved confounders. We show how, given only a latent variable model for states and actions, policy value can be identified from off-policy data. Our method involves two stages. In the first, we show how to use proxies to estimate stationary distribution ratios, extending recent work on breaking the curse of horizon to the confounded setting. In the second, we show optimal balancing can be combined with such learned ratios to obtain policy value while avoiding direct modeling of reward functions. We establish theoretical guarantees of consistency, and benchmark our method empirically.
\end{abstract}

\section{Introduction}\label{sec:intro}

A fundamental question in offline reinforcement learning (RL) is how to estimate the value of some target evaluation policy, defined as the long-run average reward obtained by following the policy, using data logged by running a \emph{different} behavior policy. This question, known as off-policy evaluation (OPE), often arises in applications such as healthcare, education, or robotics, where experimenting with running the target policy
can be expensive or even impossible,
but we have data logged following business as usual or current standards of care.
A central concern using such passively observed data is that observed actions, rewards, and transitions may be \emph{confounded} by unobserved variables, which can bias standard OPE methods that assume no unobserved confounders, or equivalently that a standard Markov decision process (MDP) model holds with fully observed state.

Consider for example evaluating a new smart-phone app to help people living with type-1 diabetes time their insulin injections
by monitoring their blood glucose level using some wearable device.
Rather than risking giving bad advice that may harm individuals, we may consider first evaluating our injection-timing policy using existing longitudinal observations of individuals' blood glucose levels over time and the timing of insulin injections. The value of interest may be the long-run average deviation from ideal glucose levels. However, there may in fact be events not recorded in the data, such as food intake and exercise, which may affect both the timing of injections and blood glucose.
Unfortunately, most previously proposed methods for OPE in RL setting do not account for such confounding, so if they are used for analysis the results may be biased and misleading.




In this work, we study OPE in an infinite-horizon, ergodic MDP with unobserved confounders, where states and actions can act as proxies for the unobserved confounders. We show how, given only a latent variable model for states and actions, the policy value can be identified from off-policy data.
We provide an optimal balancing~\citep{bennett2019policy} algorithm for estimating the policy value while avoiding direct modeling of reward functions, given an estimate of the stationary distribution ratio of states and an identified model of confounding.
In addition, we provide an algorithm for estimating the stationary distribution ratio of states in the presence of unobserved confounders, by extending recent work on infinite-horizon OPE~\citep{liu2018breaking} and efficiently solving conditional moment matching problems~\citep{bennett2019deep}.
On the theory side, we establish statistical consistency under the assumption of iid confounders, and provide error bounds for our method in close-to-iid settings.
Finally, we demonstrate that our method achieves strong empirical performance compared with several causal and non-causal baselines.

\paragraph{Notation}
We use uppercase letters such as $S$ and $X$ to denote random variables, and lowercase ones to denote nonrandom quantities.
The set of positive integers is $\Nset$, and for any $n\in\Nset$ we use $[n]$ to refer to the set $\{1,\ldots,n\}$. We denote by $\norm{\cdot}_p$ the usual functional norm, defined as $\norm{f}_p = \e[\abs{f(X)}^p]^{1/p}$, where the measure is implicit from context. Furthermore we denote as $\lpspace{p}$ the space of functions with finite $\norm{\cdot}_p$-norm. We denote by $N(\epsilon, \Fset, \norm{\cdot})$ the $\epsilon$-covering number of $\Fset$ under metric $\norm{\cdot}$, and the corresponding bracketing number by $N_{[]}(\epsilon, \Fset, \norm{\cdot})$.
Finally, for any random variable sequence $\{Q_1,Q_2,\ldots\}$, we use the notation $Q_{l:u}$ as shorthand for $(Q_l,Q_{l+1},\ldots,Q_u)$.

\section{Problem Setting}\label{sec:prob}

\begin{figure}
\centering
\includegraphics[width=0.35\textwidth]{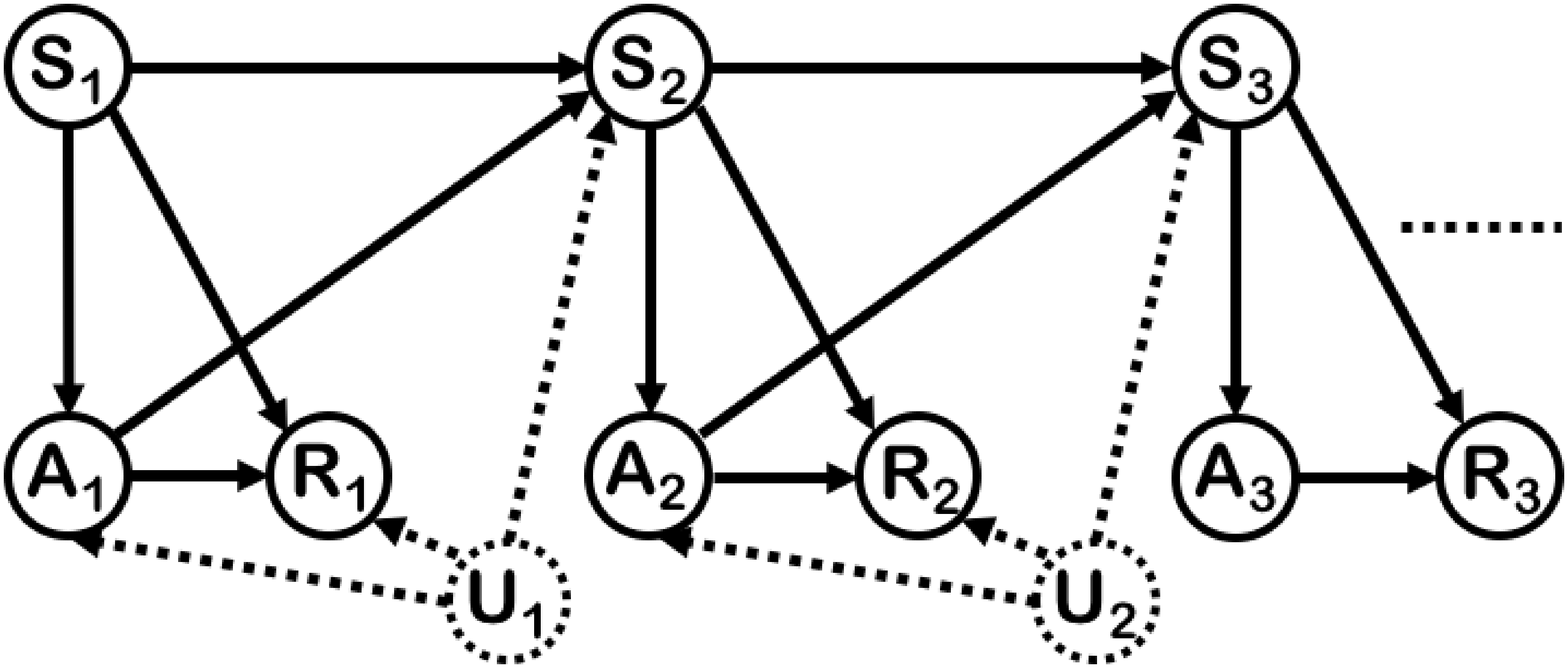}
\vspace{-3ex}
\caption{Graphical representation of the MDPUC model, in which action selection, state transition, and reward emission are confounded \textit{at every step}.}
\vspace{-1ex}
\label{fig:mdpuc}
\end{figure}

We consider the Markov Decision Process with Unmeasured Confounding, or MDPUC~\citep{zhang16markov}, which is a confounded generalization of a standard Markov decision process (MDP). An MDPUC is specified by a tuple $(\Sset, \Aset, \Uset, P_T, \mathcal R, P_0)$, where $\Sset$ is the finite state space, $\Aset = [m]$ the action space, $\Uset$ the confounder space, $P_T(s' \mid s,a,u)$ the probability of transitioning to state $s'$ from state $s$ given action $a$ and confounder $u$, $\mathcal R(s,a,u)$ the reward distribution given action $a$ was taken in state $s$ with confounder $u$, and $P_0$ the distribution over starting states. We also define $\mu_a(s,u) = \e[R(s,a,u)]$, where $R(s,a,u)$ is any random variable distributed according to $\mathcal R(s,a,u)$, and we use $S'$ to refer to the state succeeding state $S$ in a trajectory, $Z$ to refer to the triplet $(S, A, S')$, and $X$ to refer to the pair $(Z,U)$. An important assumption we make here is that the confounder values $U$ at each time step are iid, which differentiates the MDPUC setting from the more general POMDP setting. An example of this setting may be our diabetes problem from \cref{sec:intro}, where $S$ corresponds to blood glucose levels, $A$ corresponds to insulin injection decisions, $R$ is based on maintaining safe blood glucose levels, and $U$ corresponds to the exogenous unmeasured events such as food intake or exercise.\footnote{Although the confounders are likely not perfectly iid, this modelling approximation may be justified for instance if we can approximately model the confounding events by a Poisson process.}

We assume access to $N \geq 1$ trajectories of off-policy data, of lengths $T_1,\ldots,T_N$. At each time step of a trajectory we assume that we observe the state $S$, the action that was taken in that state $A$, and the corresponding reward that was received $R$. Importantly, we do \emph{not} observe the corresponding confounder value $U$. 
We assume that each trajectory was logged from a common behavior policy $\pi_b$, which depends on the confounders, where $\pi_b(a \mid s, u)$ gives the probability that $\pi_b$ takes action $a$ given state $s$ and confounder $u$. Note that although we assume our data is collected from separate trajectories, for brevity we will index our data by concatenating these trajectories together and using indices $i \in [n]$, where $n = \sum_{i=1}^N T_i$, and we denote the observed data by $\Dataset = \{Z_i, R_i\}_{i \in [n]}$.

Our task is to estimate the value of some evaluation policy $\pi_e$,
which follows the same semantics as $\pi_b$, and whose actions may optionally depend on the confounders $U$ (for simplicity, even in the case that its actions depend on $S$ only, we still use the notation $\pi_e(a\mid s,u)$)
We make the following ergodicity and mixing assumptions about the behavior and evaluation policies.

\begin{assumption}[Mixing]
\label{asm:mixing}For some $2 < p \leq \infty$ we have 
$\sum_{k=1}^\infty k^{2/(p-2)} \beta(k) < \infty$,
where $\beta(k)$ are the $\beta$-mixing coefficients of the Markov chain of $X$ values induced by $\pi_b$.
\end{assumption}

\begin{assumption}[Ergoicity]
\label{asm:ergodic}
The Markov chain of $X$ values under each of $\pi_b$ and $\pi_e$ is ergodic. Furthermore, the chain of $X$ values under $\pi_b$ is stationary.
\end{assumption}

\Cref{asm:mixing} uses $\beta$-mixing coefficients, which quantify how close to independent $X$ values $k$ steps removed are in the Markov chain, with coefficients of zero implying independence. In our stationary Markovian setting these are defined according to $\beta(k) = \sup_{B \in \sigma(X_{k+1})} \abs{P(B \mid \sigma(X_1)) - P(B)}$, where $\sigma(X)$ denotes the $\sigma$-algebra generated by $X$.

\Cref{asm:ergodic} implies that the $X$ values obtained from each policy have a unique stationary distribution, and under $\pi_b$ all values follow this stationary distribution. We let $\e_b$ and $\e_e$ denote expectations taken with respect to these stationary distributions, and assume that probability statements refer to the stationary distribution under $\pi_b$ where not specified.
In addition, we will use the notation $d(Q)$ to denote the stationary density ratio under $\pi_e$ versus $\pi_b$, for any random variable $Q$ that is measurable with respect to $X$.\footnote{That is, for any such $Q$, we define $d(Q)$ such that $\e_{e}[g(Q)] = \e_{b}[d(Q) g(Q)]$ for any measurable function $g$. Note that this involves slight abuse of notation since the function $d$ depends on the random variable $Q$.}
Given this, we define the \emph{value} of $\pi_e$ to be $\pv(\pi_e) = \e_e[\mu_A(S,U)].$

Finally, we make the following regularity assumptions about the MDPUC, which are all standard kinds of assumptions that are easily satisfied in real settings where rewards and states are bounded.
\begin{assumption}[State Visitation Overlap]
\label{asm:overlap}
$\norm{d(S)}_{\infty} < \infty$.  
\end{assumption}
\begin{assumption}[Bounded Reward Variance]
\label{asm:bounded-variance} $
\var[R \mid Z] \leq \sigma^2 ~\text{and}~ \var[R \mid Z, U] \leq \sigma^2 ~\text{almost surely}.
$
\end{assumption}
\begin{assumption}[Bounded Mean Reward]
\label{asm:bounded-mean}
For each $a$, $\mu_a(S,U)$ is uniformly bounded almost surely.
\end{assumption}

\section{Related Work}

The infinite-horizon OPE problem has received fast-growing interest recently~\citep{liu2018breaking,gelada19off,kallus19efficiently,nachum19dualdice,mousavi20blackbox,uehara19minimax,zhang20gendice}.  Most of these approaches are  based on some form of moment matching condition, derived from the stationary distribution of the corresponding Markov chains, and can avoid the exponential growth of variance in typical importance sampling methods~\citep{liu2018breaking}.  Our work extends this research to a more general setting with unobserved confounders.
Similar to our work, \citet{tennenholtz20off} has addressed OPE under unmeasured confounding in the POMDP setting, however their work relies on complex invertibility assumptions and is limited to tabular settings.
In addition there is a tangential line of work investigating the limitations of OPE under unmeasured confounding in nonparametric settings and constructing partial identification bounds \citep{kallus20confounding,namkoong2020off}, which differs from our focus on specific settings where the model of confounding is identifiable and therefore so is the policy value.
Furthermore OPE under unmeasured confounding has been studied in contextual bandit settings~\citep{bennett2019policy}, which may be viewed as a special case of our problem where states are generated iid in every step. 

Related to the \emph{evaluation} problem is policy \emph{learning}, where the goal is to interact with an unknown environment to optimize the policy.  The partially observable MDP (POMDP) is a classic model for sequential decision making with unobserved state~\citep{kaelbling98planning}, and has been extensively studied~\citep{spaan12partially,azizzadenesheli16reinforcement}.  More recently, a few authors have applied counterfactual reasoning techniques to RL (including multi-armed bandits)~\citep{bareinboim15bandits,zhang16markov,lu18deconfounding,buesing19woulda}.
While evaluation might appear simpler than learning, OPE methods only have access to a fixed set of data and cannot explore.
This restriction leads to different challenges in algorithmic development that are tackled by our proposed method. 

Finally, another related area of research is on using proxies for true confounders \citep{wickens1972note,frost1979proxy}. Much of this work involves fitting and using latent variable models for confounders, or studying sufficient conditions for identification of these latent variable models \citep{cai2008identifying,wooldridge2009estimating,pearl2012measurement,kuroki2014measurement,edwards2015all,louizos2017causal,kallus2018causal}.
This is complementary to our work, since we propose an estimator that uses a latent variable model for confounders, but do not study how to fit it.


\section{Theory for Optimally Weighted Policy Evaluation}\label{sec:alg}




In this work, we consider generic weighted estimators of the form 
\begin{equation}
  \label{eq:weighted-estimator}
  \hat\tau_W = \frac{1}{n} \sum_{i=1}^n W_i R_i,
\end{equation}
where $W = W_{1:n}$ is any vector of weights that is measurable with respect to $Z_{1:n}$. Inspired by \citet{kallus2018balanced} and \citet{bennett2019policy}, we proceed by deriving an upper-bound for the risk of policy evaluation. First, we observe that the value of $\pi_e$ is given by
\begin{equation*}
\pv(\pi_e) =  \sum_{a=1}^m \e_e [\pi_e(a \mid S, U) \mu_a(S,U)] = \sum_{a=1}^m \e_b [d(S) \pi_e(a \mid S, U) \mu_a(S,U)],
\end{equation*}
where the second equality follows from the observation that $d(S,U) = d(S)$ under the iid confounder assumption.
In addition, it is easy to verify that $\e_b [W R] = \sum_{a=1}^m \e_b [W \delta_{A a} \mu_a(S,U)]$.
This suggests that if we knew $U_{1:n}$, the bias of balanced policy evaluation could be approximated by $\frac{1}{n}\sum_{i=1}^n \sum_{a=1}^m (W_i \delta_{A_i a} - d(S_i) \pi_e(a \mid S_i, U_i)) \mu_a(S_i, U_i)$, which motivates the following theorem.

\begin{theorem}
\label{thm:mse-bound}
For any vector $W$, vector-valued function $g = (g_1,\ldots,g_m)$, and constant $\lambda$, define
\begin{align*}
  f_{ia} &= W_i \delta_{A_i a} - d(S_i) \pi_e(a \mid S_i, U_i)\,, \\
  B(W,g) &= \frac{1}{n} \sum_{i=1}^n \sum_{a=1}^m \e[f_{ia} g_a(S_i, U_i) \mid Z_i]\,, \\
  J_\lambda(W, g) &= B(W,g)^2 + \frac{\lambda}{n^2} \norm{W}^2\,.
\end{align*}
Then, if $\lambda \geq 4\sigma^2$ and $J_\lambda(W,\mu) = O_p(r_n)$, where $\mu = (\mu_1,\ldots,\mu_m)$ are the true mean reward functions, it follows from \cref{asm:ergodic,asm:mixing,asm:overlap,asm:bounded-variance,asm:bounded-mean} that $\hat\tau_W = \pv(\pi_e) + O_p(\max(n^{-1/2}, r_n^{1/2}))$.
\end{theorem}

This result suggests finding weights $W$ in \cref{eq:weighted-estimator} that minimize $\sup_{g \in \Gset} J_\lambda(W,g)$ for some vector-valued function class $\Gset$, since if $\mu \in \Gset$ and we can minimize this upper bound uniformly over $\Gset$ at an $O_p(1/n)$ rate, then $\hat\tau_W$ is $O_p(1 / \sqrt{n})$-consistent for $\pv(\pi_e)$.

Next, we describe regularity assumptions about the function class $\Gset$ under which the above $O_p(1/\sqrt{n})$ convergence is achievable. In describing these assumptions, we assume that the space $\Gset$ is normed, and we define $\Gset^* = \{g / \norm{g} : g \in \Gset, \norm{g} > 0\}$, and $\Gset^*_a = \{g_a : \exists (g'_1,\ldots,g'_m) \in \Gset^* \text{ with } g_a = g'_a\}$. 

\newcommand{\minusspace}{} 

\begin{assumption}[Compactness]
\label{asm:compactness}
$\Gset$ and $\Gset^*$ are compact.
\end{assumption}
\minusspace

\begin{assumption}[Convexity]
\label{asm:convexity}
$\Gset$ is convex.
\end{assumption}
\minusspace

\begin{assumption}[Symmetry]
\label{asm:symmetry}
$g \in \Gset \iff -g \in \Gset$.
\end{assumption}
\minusspace

\begin{assumption}[Continuity]
\label{asm:continuity}
$h_g(z)$ and $k_g(z)$ are continuous in $g$ for every $z \in \Sset \times \Aset \times \Sset$, and $g_a(s,u)$ is continuous in $s$ and $u$ for every $g \in \Gset$.
\end{assumption}
\minusspace

\begin{assumption}[Uniformly Bounded Functions]
\label{asm:bounded-g}
There exists a constant $0 < G < \infty$ such that for every $g \in \Gset^*$, $a \in [m]$, $s \in \mathcal S$, and $u \in \mathcal U$ we have $g_a(s,u) \leq G$.
\end{assumption}
\minusspace

\begin{assumption}[Uniform Bracketing Entropy]
\label{asm:uniform-entropy}
$\int_0^\infty \sqrt{\log N_{[]}(\epsilon, \Gset_a^*, \lpspace{p})} d \epsilon < \infty$ for each $a \in [m]$, where $p$ takes the same value as in \cref{asm:mixing}.
\end{assumption}
\minusspace

\begin{assumption}[Non-degeneracy]
\label{asm:non-degenerate} $\sup_{g \in \Gset^*} P(\e[g_A(S,U) \mid Z] = 0) < 1.$
\end{assumption}

\Cref{asm:compactness,asm:convexity,asm:symmetry,asm:continuity,asm:bounded-g,asm:uniform-entropy} are purely technical assumptions about $\mathcal G$ only, and hold for many commonly used function classes.
In particular, we provide the following lemma, which justifies that they hold for a variety of Reproducing Kernel Hilbert Spaces (RKHSs).

\begin{lemma}
\label{lem:kernel-g}
Let $K$ be a symmetirc, PSD, continuous, and bounded kernel, and let $\norm{g}^2 = \sum_{a=1}^m \norm{g_a}^2_K$, where $\norm{\cdot}_K$ is the RKHS norm with kernel $K$. Then for any $\gamma>0$, the function class $\mathcal G_{K,\lambda} = \{g : \norm{g} \leq \gamma\}$ satisfies \cref{asm:compactness,asm:convexity,asm:symmetry,asm:continuity,asm:bounded-g,asm:uniform-entropy}.
\end{lemma}

Finally, \cref{asm:non-degenerate} is used to avoid the pathological situation where $\e[g_A(S,U) \mid Z] = 0$ almost surely for some non-zero $g$, in which case $B(W,g) = B(W',g)$ for any $W,W' \in \mathbb R^n$ and bias cannot be controlled. Note that this is a joint assumption on the class $\Gset$ and the data generating process, and is similar to identifiability conditions in other causal inference works with latent variable such as in \citet{miao2018identifying}; it can be seen as the assumption that any $\mu,\mu' \in \Gset$ with $\mu \neq \mu'$ would induce a different observed distribution of data.

\begin{theorem}
\label{thm:bound-convergence}
Given \cref{asm:ergodic,asm:mixing,asm:compactness,asm:convexity,asm:symmetry,asm:continuity,asm:bounded-g,asm:uniform-entropy,asm:non-degenerate}, $\inf_{W \in \mathbb R^n} \sup_{g \in \Gset} J_{\lambda}(W, g) = O_p(1/n).$
\end{theorem}


\subsection{Sensitivity to Nuisance Estimation Error and Model Misspecification}
\label{sec:realistic-setting}

Next we extend our theory to more realistic settings, and consider the effects of estimation errors and non-iid confounding. We present simplified results here for the common case where $\pi_e$ is measurable with respect to $S$ only, and present results for the more general case where $\pi_e$ can also depend on $U$ in \cref{apx:sensitivity-theory}.
For this analysis, we let some normed function class $\mathcal F$ be given. Then, for any measures $p$ and $q$ on $\mathcal U$ we define the integral probability metric $D_{\mathcal F}(p,q) = \sup_{\norm{f}_{\mathcal F} \leq 1} \abs{\int f(u) dp(u) - \int f(u) dq(u)}$,\footnote{Examples include total variation distance, where $\norm{f}_{\mathcal F} = \norm{f}_{\infty}$, the maximum mean discrepancy where $\norm{f}_{\mathcal F}$ is given by some RKHS norm, and Wasserstein distance, where $\norm{f}_{\mathcal F}$ is given by the Lipschitz norm.} and we make the following additional assumptions.

\begin{assumption}
\label{asm:f-norm}
There exists some constant $F$, such that for every $g \in \mathcal G$ and $a \in [m]$ we have $\norm{g_a}_{\mathcal F} \leq F$ and $\norm{\mu_a}_{\mathcal F} \leq F$.\footnote{We note that given \cref{asm:bounded-mean,asm:bounded-g}, \cref{asm:f-norm} can be satisfied using supremum norm, which corresponds to $D_{\mathcal F}$ being total variation distance. However, we choose to make the theory more flexible and allow for weaker distributional metrics, since this may make \cref{lem:estimation-error,thm:non-iid-bound} easier to satisfy.}
\end{assumption}

\begin{assumption}
\label{asm:correct-specification}
$\mu \in \mathcal G$, and $\lambda \geq 4\sigma^2$, where $\sigma$ is defined in \cref{asm:bounded-variance}.\footnote{We note that the second part of this assumption is easily satisfied, since $J_{\lambda}(W,g) = J_{4 \sigma^2}(W, (4 \sigma^2 / \lambda)^{1/2} g)$, so using the ``wrong'' $\lambda$ is equivalent to using $\lambda=4\sigma^2$ and a different $\mathcal G$ radius.}
\end{assumption}

We first address the issue that the adversarial objective considered above depends on the conditional density of $U$ given $Z$, and the state density ratio $d$. In practice these both would usually need to be estimated from data. Let $\varphi(z)$ and $\hat\varphi(z)$ denote the true and estimated conditional distribution of $U$ respectively given $Z=z$, let $\hat d$ be the estimated state density ratio. In addition let $\hat J_{\lambda}(W,g)$ be the objective using $\hat \varphi$ and $\hat d$ in place of $\varphi$ and $d$, and let $W^* = \argmin_W \sup_{g \in \mathcal G} \hat J(W,g)$. 



\begin{lemma}
\label{lem:estimation-error}
Suppose that $D_{\mathcal F}(\varphi(Z_i), \hat\varphi(Z_i)) = O_p(r_n)$, and $\abs{d(S_i) - \hat d(S_i)} = O_p(r_n)$, for every $i \in [n]$. Then, given \cref{asm:ergodic,asm:mixing,asm:overlap,asm:bounded-mean,asm:bounded-variance,asm:compactness,asm:convexity,asm:symmetry,asm:continuity,asm:bounded-g,asm:uniform-entropy,asm:non-degenerate,asm:f-norm,asm:correct-specification}, we have $\hat\tau_{W^*} = \pv(\pi_e) + O_p(\max(n^{-1/2}, n^{1/2} r_n^2))$.
\end{lemma}

\Cref{lem:estimation-error} implies that our methodology will be consistent as long as $\varphi$ and $d$ are estimated at a $o_p(n^{-1/4})$ rate, and that we can still obtain $O_p(n^{-1/2})$-consistency if $\varphi$ and $d$ are estimated at a $O_p(n^{-1/2})$ rate. We discuss the estimation of $\varphi$ and conditions under which the required rates are obtainable in \cref{apx:phi-estimation}, and the estimation of $d$ in \cref{sec:state-density-ratio}.

Next, we consider minor violations in the iid confounder assumption of the MDPUC model. Specifically, we consider an alternate model where $U$ values form a Markov chain. Under this alternate model, we provide the following theorem bounding the squared error. 

\begin{theorem}
\label{thm:non-iid-bound}
Suppose that the assumptions of \cref{lem:estimation-error} hold, and $\norm{d(S,U)}_{\infty} < \infty$. In addition let $\varphi_i$ and $\varphi_i^*$ denote the conditional densities of $U_i$ given $Z_i$ and $Z_{1:n}$, let $b = \max_a \norm{\mu_a}_{\infty}$, and let $c = \sqrt{2} F (1 + \norm{W^*}^2/n)^{1/2}$. Then we have $(\hat\tau_{W^*} - \pv(\pi_e))^2 \leq \epsilon^2 + O_p(\max(1/n, n r_n^4))$, where

\begin{equation*}
    \abs{\epsilon} \leq c \left( \frac{1}{n} \sum_{i=1}^n D_{\mathcal F}(\varphi_i, \varphi_i^*)^2 \right)^{1/2} + b \norm{d(S,U) - d(S)}_2 + O_p(\max(n^{-1/2}, n^{1/2} r_n^2)).
\end{equation*}
\end{theorem}
We note that in the iid confounder case $\varphi_i = \varphi_i^*$ and $d(S,U)=d(S)$, so the first two terms disappear, and the result reduces to that of \cref{lem:estimation-error}. In addition under \cref{asm:bounded-mean}, the constant $b$ must be finite. Therefore \cref{thm:non-iid-bound} allows us to bind the asymptotic bias in ``near-iid'' settings, where the terms $D_{\mathcal F}(\varphi_i,\varphi_i^*)$ and $\norm{d(S) - d(S,U)}_2$ are small. We provide more detail and discussion, and a tighter version of this bound in \cref{apx:sensitivity-theory}.

Finally, we note that following the same argument as \citet{bennett2019policy}, if $\Gset$ is universally approximating then we can still ensure consistency even if $\mu \notin \Gset$, although possibly at a rate slower than $O_p(\max(n^{-1/2}, n^{1/2} r_n^2))$. We refer readers to \cref{apx:universally-approximating} for details.







\section{Methodology}
\label{sec:methodology}

We now discuss how the optimal balancing estimator analyzed in \cref{sec:alg} above can be realized. There are three steps to implementing such an estimator: (1) estimating the conditional distribution of $U$ given $Z$ (denoted by $\varphi$); (2) estimating $d$; and (3) calculating $W^* = \argmin_W \sup_{g \in \mathcal G} \hat J_{\lambda}(W,g)$. We focus only on the second two parts, since the first has been extensively studied in past work.


\subsection{Estimating the Stationary Density Ratio}
\label{sec:state-density-ratio}

Here, we pose learning the stationary density ratio $d(S)$ as a conditional moment matching problem. Similarly to \citet{liu2018breaking}, we can identify $d$ via a set of moment restrictions, as follows. 
\begin{theorem}
\label{thm:stationary-density-ratio}
Let $\beta(z) = \e[\pi_e(A \mid S,U) / \pi_b(A \mid S,U) \mid Z=z]$. Then under \cref{asm:ergodic}, the stationary density ratio $d(S)$ is the unique function satisfying the regular moment condition $\e[d(S)] = 1$, as well as the conditional moment restriction $\forall S':~ d(S') = \e_b[d(S) \beta(Z) \mid S']$.
\end{theorem}

Motivated by past work on efficiently solving conditional moment matching problems \citep{bennett2019deep}, we propose to estimate $d$ by solving a smooth-game optimization problem. Let
$M(Z;d,h,c) = h(S') (d(S) \beta(Z) - d(S')) + c (d(S) - 1)$, and $U_n(d,\tilde d, h, c) = (1/n) \sum_{i=1}^n ( M(Z_i;d,h,c) - (1/4) M(Z_i;\tilde d, h,c)^2 )$.
Then given some prior estimate $\tilde d$ of $d$, which might come from a previous GMM estimate or some other methodology, and function classes $\mathcal D$ and $\mathcal H$, our proposed estimator takes the form
\begin{equation}
\label{eq:d-estimator}
  \hat d = \argmin_{d \in \Dset} \sup_{h \in \Hset, \abs{c} \leq \lambda_c} U_n(d, \tilde d, h, c).
\end{equation}

We note that the choice of function classes $\mathcal D$ and $\mathcal H$, and the value $\lambda_c$ are all hyperparameter choices. This approach generalizes that of \citet{bennett2019deep}, which was originally developed for solving the conditional moment matching problem for instrumental variable regression. We discuss the derivation of this algorithm in \cref{apx:gmm-derivation}, and discuss known results on the rate of convergence of such GMM estimators in \cref{apx:d-estimation}.


In practice, we can start with an initial guess for $\tilde d$ (such as $\tilde d(s) = 1 \ \forall s$), and then iteratively solve \cref{eq:d-estimator} with $\tilde d$ being the previous solution. In addition, for our experiments we choose to use norm-bounded RKHSs for $\mathcal D$ and $\mathcal H$, which allows the optimization to be performed analytically (details are given in \cref{apx:representer-details}).\footnote{This is in contrast to \citet{bennett2019deep}, who used neural networks and smooth-game optimization techniques for their instrumental variable regression estimator.} Finally, since $\beta$ is unknown we can estimate it using $\hat\varphi$.

\subsection{Solving for Optimal Weights}
\label{sec:solving-w-qprog}

We now describe a method for analytically computing $\argmin_W \sup_{g\in\Gset} \hat J_\lambda(W,g)$ for kernel-based classes $\Gset = \Gset_{K,\lambda}$, as defined in \cref{lem:kernel-g}. Our approach is based on the following theorem.

\begin{theorem}
\label{thm:w-objective}
For each $i \in [n]$, let $\tilde U_i$ be a shadow variable which is iid to $U_i$ given $Z_i$, and define
\begin{align*}
G_{ij} &= \frac{1}{n^2} \left( \delta_{A_i A_j} \e_{\hat\varphi}[K((S_i,U_i),(S_j,\tilde U_j)) \mid Z_i, Z_j] + \lambda \delta_{i j} \right) \\
g_i &= \frac{1}{n^2} \sum_{j=1}^n \hat d(S_j) \e_{\hat\varphi}[ \pi_e(A_i \mid S_j, \tilde U_j) K((S_i,U_i),(S_j,\tilde U_j)) \mid Z_i, Z_j],
\end{align*}
where $\e_{\hat\varphi}$ denotes expectation under the estimated conditional distribution given by $\hat\varphi$. Then for some $C$ that is constant in $W$, we have we have $\sup_{g \in \Gset_K} \hat J_\lambda(W,g) = W^T G W - 2 g^T W + C.$
\end{theorem}

First, using our estimated posterior $\hat\varphi$ and stationary density ratio $\hat d$ we can compute $G$ and $g$. Then $W^* = \argmin_W \sup_{g \in \Gset_K} \hat J_\lambda(W,g)$ is given by $G^{-1} g$.\footnote{If we wish to impose some constraints on $W$, such as $W \in \Delta^n$ (the set of categorical distributions over $n$ categories), then we could instead solve a quadratic program. However our theory does not support this, and in practice in our experiments we calculate $W^*$ using the unconstrained analytic solution.}
Finally, we note that in the case that $\Sset$ and $\Uset$ are discrete, as in our experiments, we can calculate $W^*$ more efficiently, constructing a matrix of order $n_d \times n_d$ rather than $n \times n$, where $n_d \le \min\{n, m\abs{\Sset}^2\}$ is the number of \emph{distinct} $(S,A,S')$ tuples in the training data. 
Details are provided in \cref{apx:weights-solving-details}.


\section{Experiments}\label{sec:exp}

We now empirically evaluate our proposed method and demonstrate its benefits over state-of-the-art baselines for OPE.
Our method requires as input an approximate confounder model, $\hat\varphi$, for the posterior of $U$ given $Z$. That is, $\hat\varphi(z) \approx P(\cdot \mid Z=z)$.
Since our baseline methods cannot account for unmeasured confounding, for fairness we allow these methods access to $\hat\varphi$. 
Specifically, for each $i \in [n]$ we sample a value $\hat U_i$ from the approximate posterior $\hat\varphi(Z_i)$, and augment the baselines' data with $\{\hat U_i\}_{i \in [n]}$. They then use $(S_i, \hat U_i)$ as the state variable rather than $S_i$. However, since we only assumed a latent variable model for $(Z,U)$, but not for $(Z,U,R)$ (that is, we do not assume an outcome model) we expect that this may still lead to biased estimators even if $\hat\varphi$ is perfect.\footnote{This is because we can only sample confounders conditioned on $Z$, not on $(Z,R)$, so the dataset augmented with imputed confounders will be distributed differently to a dataset augmented with the true confounders.} We consider the following baselines: \textbf{Direct Method} which fits an outcome model using the imputed confounders; \textbf{Doubly Robust} which combines our optimal balancing weights with the Direct Method, by re-weighting the estimated reward residuals; \textbf{Inverse Propensity Scores (IPS)} with IPS weights calculated as in \citet{liu2018breaking}; and \textbf{Black-Box} which is state-of-the-art recently proposed weighted estimator \citep{mousavi20blackbox}. For a detailed description of these baselines see \cref{apx:baselines}, and for additional details on hyperparameters for our method and baselines see \cref{apx:hyperparameters}.

\begin{figure}
\centering
\includegraphics[width=0.25\textwidth]{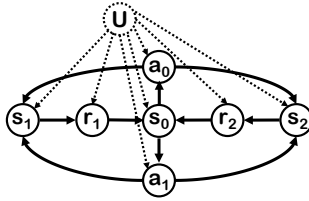}
\vspace{-3ex}
\caption{The C-ModelWin environment.}
\vspace{-1ex}
\label{fig:modelwin}
\end{figure}

\textbf{First Experiment.}
In this experiment we consider the C-ModelWin environment, which is a confounded variant of ModelWin~\citep{thomas16data}. This is a simple tabular environment with 3 states, 2 actions, and 2 confounder levels. We depict this environment in \cref{fig:modelwin}, and describe it in detail in \cref{apx:env}.

First we compared our estimator $\hat\tau_W$, with $\hat d$ calculated as in \cref{sec:state-density-ratio} and $W$ estimated as in \cref{sec:solving-w-qprog}, against the baselines. For this comparison we used the true conditional confounder distribution for $\hat\varphi$, with datasets of varying number of trajectories of length 100, and performing 50 repetitions for each configuration of estimator and number of trajectories to compute 95\% confidence intervals.\footnote{We also used these trajectory lengths and numbers of repetitions in our sensitivity experiments.} We display the results of this comparison in the first two plots of \cref{fig:modelwin_ours}, where we plot the estimated policy value and corresponding root mean squared error (RMSE) respectively for every configuration. We see that our estimator achieves strong results, with near-zero bias as we increase the number of trajectories. This is in contrast to the baselines, all of which converge to biased estimates as we increase the number of trajectories, with significantly higher RMSE.

Next, we investigated the sensitivity of our estimator to the assumption of iid confounders. Let $\mathcal{P}_{\text{iid}}$ denote the iid confounder distribution under the C-ModelWin environment, and $\mathcal{P}_{\text{alt}}$ denote some alternative distribution, where within each trajectory the distribution of the confounder at time $t$ depends on the confounder at time $t-1$. We experimented with a variation of C-ModelWin, where confounders were distributed according to $\alpha \mathcal{P}_{\text{iid}} + (1-\alpha) \mathcal{P}_{\text{alt}}$, for some $\alpha \in [0, 1]$. This means when $\alpha=1.0$ we recover C-ModelWin, and as we decrease $\alpha$ the iid confounder assumption becomes increasingly violated. The specific alternative model $\mathcal P_{\text{alt}}$ used is described in \cref{apx:model-misspecification}. We display the RMSE of our estimator for various numbers of trajectories and various values of $\alpha$ in the third plot in \cref{fig:modelwin_ours}. We see here that, as predicted in \cref{sec:realistic-setting}, the effects of this assumption violation are continuous; when $\alpha$ is close to one the RMSE only increases slightly. 

Thirdly, we investigated the effects of introducing error into $\hat\varphi$. We injected error by adding random Gaussian noise of varying variance to the logits of the conditional confounder distribution for each level of $Z$ (before re-normalizing) and measured the amount of noise via the average standard deviation (ASD) metric, which calculates the expected standard deviation of $\hat P(U=u \mid Z)$, averaged over the levels of $U$.\footnote{With expectation taken over $Z$, and standard deviation over random noise injection.}  Details of this metric and noise injection are in  \cref{apx:posterior-noise}. We display the RMSE of our estimator under varying levels of noise injection in the fourth plot of \cref{fig:modelwin_ours}. We observe that again, as predicted in \cref{sec:realistic-setting}, the effects of noise injection are continuous; as we increase the level of noise injection (as measured by ASD) the RMSE gradually increases, with minimal impact when the error in $\hat\varphi$ is small. Finally, we provide additional plots in \cref{apx:additional-plots}, repeating both sensitivity experiments for the baselines. 

\begin{figure}[H]
\centering
\subfigure{\includegraphics[width= 0.24\textwidth]{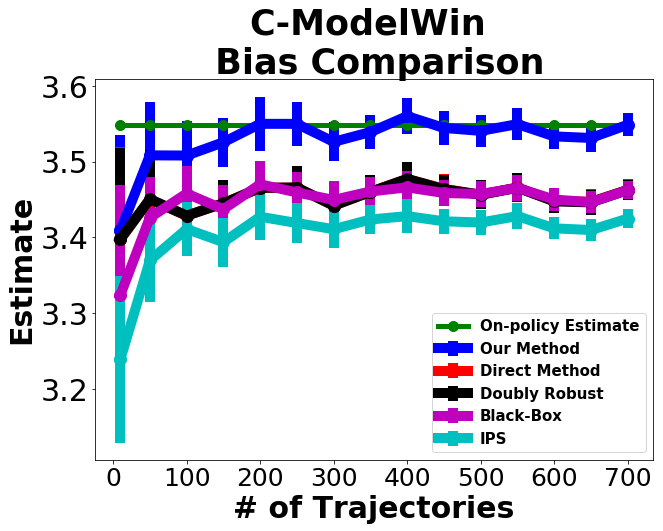}}
\subfigure{\includegraphics[width= 0.24\textwidth]{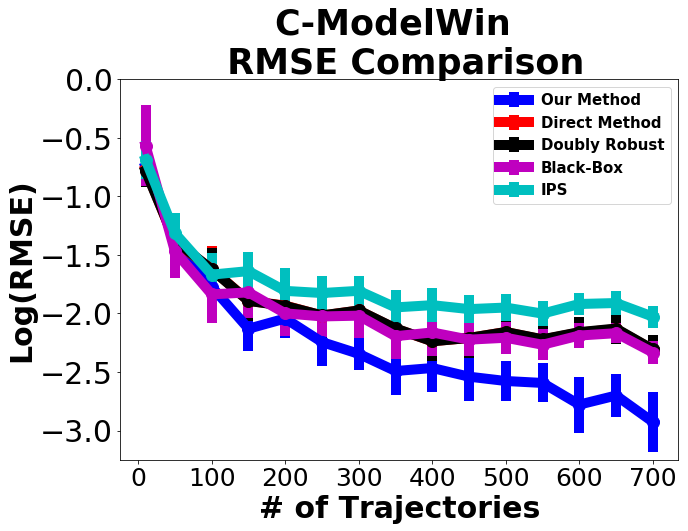}}
\subfigure{\includegraphics[width= 0.24\textwidth]{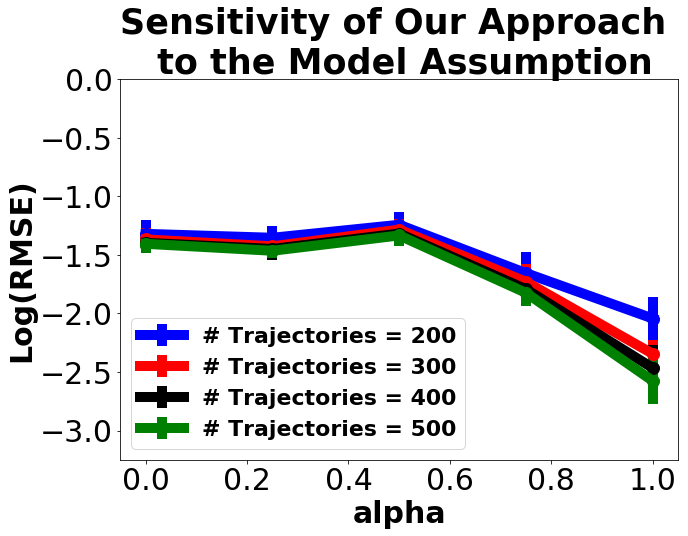}}
\subfigure{\includegraphics[width= 0.24\textwidth]{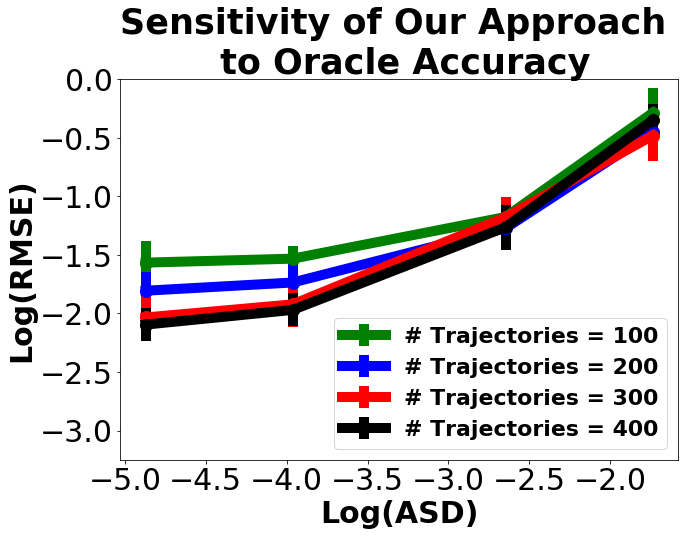}}
\caption{C-ModelWin Results. From left to right: The off-policy estimate, The $\log({\rm RMSE})$ of different methods as we change the number of trajectories, sensitivity of our estimator to model misspecification, and to noise in the confounders posterior distribution.}
\label{fig:modelwin_ours}
\end{figure}


\begin{figure}
\centering
\includegraphics[width=0.2\textwidth]{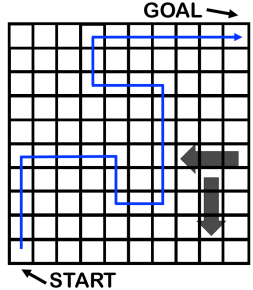}
\vspace{-3ex}
\caption{The GirdWorld environment.}
\vspace{-1ex}
\label{fig:grid}
\end{figure}

\textbf{Second Experiment.} In this experiment we consider a confounded version of the GridWorld environment. This environment consists of a $10 \times 10$ grid, with 4 actions corresponding to attempted movement in each direction, reward based on moving toward the goal, and 2 confounder levels. We depict this envirnment in \cref{fig:grid}, and describe it in detail in \cref{apx:env}.

We performed the same set of experiments with GridWorld as with C-ModelWin, except that each trajectory was of length 200. We detail the alternative non-iid confounder model used in the sensitivity part of the experiment in \cref{apx:model-misspecification}, we display the corresponding plots in \cref{fig:grid_ours}, and we include additional sensitivity results for baselines in \cref{apx:additional-plots}. In general our results here follow the same trend as in the previous experiment. We note that with GridWorld, which is much more complex than C-ModelWin, the benefits of our methodology are even more evident, with a larger relative decrease in RMSE compared to baselines. Interestingly, in this setting we see that our method seems especially robust to model assumption violations and nuisance error, with relatively small increases in RMSE in the second two plots. We hypothesize that this is because the setting is more challenging than C-ModelWin, so the error introduced by these perturbations is relatively small compared with the overall errors of the estimators. This suggests that our estimator may be relatively robust to these issues in challenging real-world settings where RMSE is naturally relatively high.

\begin{figure}[H]
\centering
\subfigure{\includegraphics[width= 0.24\textwidth]{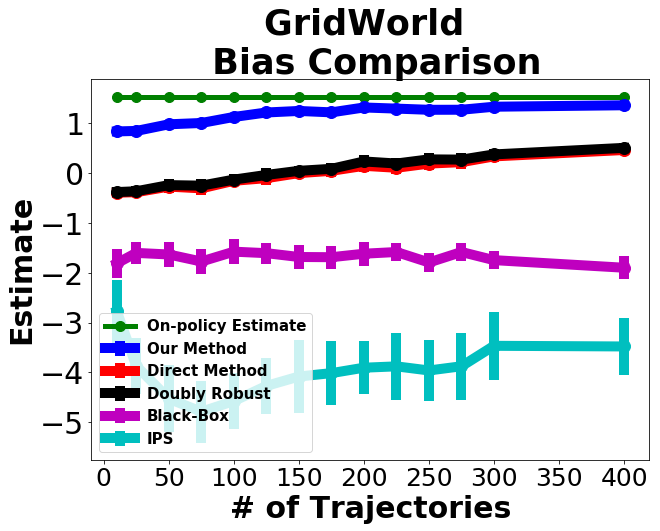}}
\subfigure{\includegraphics[width= 0.24\textwidth]{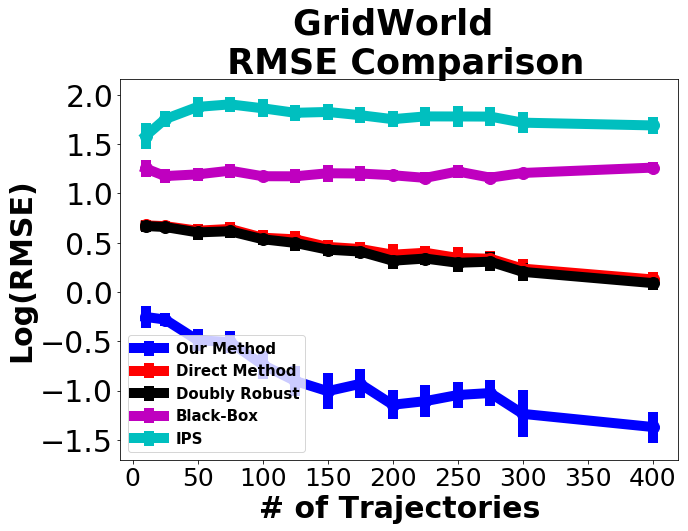}}
\subfigure{\includegraphics[width= 0.24\textwidth]{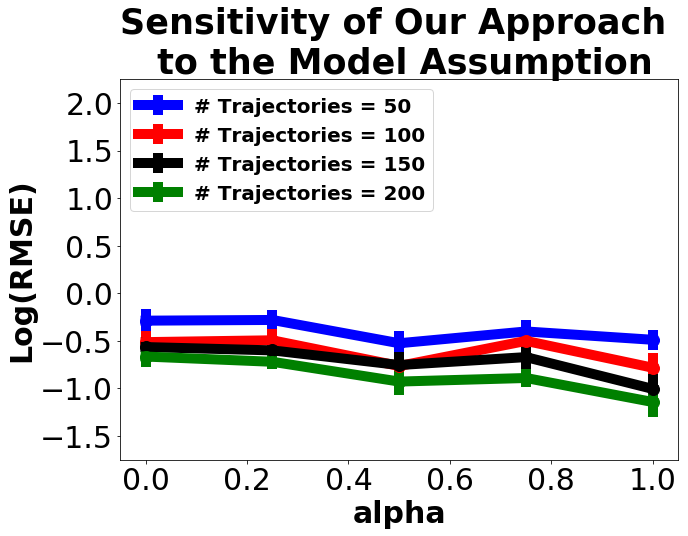}}
\subfigure{\includegraphics[width= 0.24\textwidth]{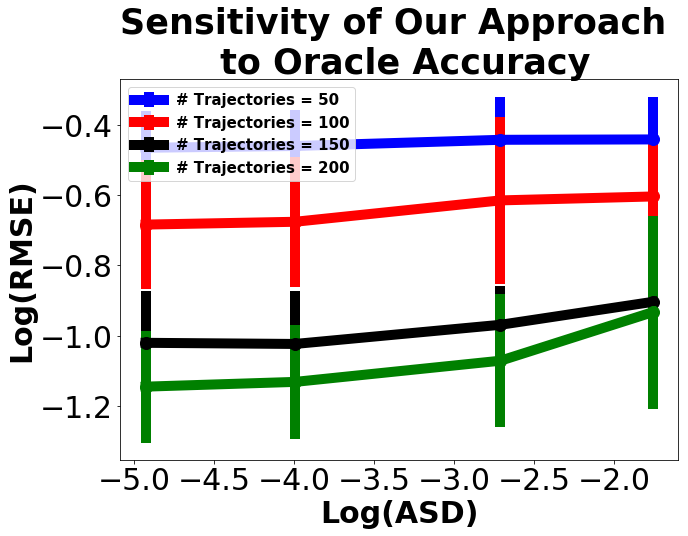}}
\caption{Confounded GridWorld Results. From left to right: The off-policy estimate, The $\log({\rm RMSE})$ of different methods as we change the number of trajectories, sensitivity of our estimator to model misspecification, and to noise in the confounders posterior distribution. Note that we changed the $y$ axis scale in the last plot for clarity, since the effect was very small.}
\label{fig:grid_ours}
\end{figure}

\section{Conclusion}

In this work, we considered OPE in infinite-horizon reinforcement learning with unobserved confounders.  We proposed a novel estimator, and showed its consistency under proper assumptions.  This is in contrast to existing estimators designed for fully-observable MDPs, which typically are unbiased and inconsistent.  We also validated our method empirically, demonstrating its accuracy against baselines and corroborating the theoretical analysis. These promising results open up a number of interesting research directions such as improving accuracy with the doubly robustness augmentation~\citep{kallus19double,tang20doubly}, or avoiding the dependency on the knowlege of behavior policy by using black-box or behavior-agnostic methods~\citep[e.g.,][]{nachum19dualdice,mousavi20blackbox}. Last but not least, one could apply this methodology to the problem of policy optimization, using a fixed set of behavior policy data with unmeasured confounders.

\bibliography{main}
\bibliographystyle{iclr2020_conference}

\clearpage
\onecolumn
\appendix

\section{Additional Lemmas}

\begin{lemma}
\label{lem:mixing}
Let \cref{asm:mixing} be given, and define
\begin{align*}
  V = \left( \frac{1}{n} \sum_{i=1}^n f(X_i) \right)^2,
\end{align*}
where $\e[f(X)] = 0$, and $\norm{f(X)}_{\infty} < \infty$. Then $\e[V] = O(1/n)$.
\end{lemma}

\begin{proof}
Given the assumption that $\e[f(X)] = 0$, it is clear that
\begin{equation*}
  \e[V] = \frac{1}{n^2} \sum_{i=1}^n \var[f(X_i)] + \frac{2}{n^2} \sum_{i=1}^n \sum_{j=1}^{i-1} \covar[f(X_i), f(X_j)].
\end{equation*}

Now by assumption $\norm{f(X)}_{\infty} < \infty$, thus $\var[f(X)]$ is finite so the first term in the above expression is in $O(1/n)$. Thus it remains to bound the second term. Next we note that the $X$ values are independent between trajectories, thus we can partition this term according to
\begin{equation*}
  \frac{2}{n^2} \sum_{t=1}^N \sum_{i=1}^{T_t} \sum_{j=1}^{i-1} \covar[f(X^{(t)}_i), f(X^{(t)}_j)],
\end{equation*}
where $X^{(t)}_i$ denotes the $i$'th observation of the $t$'th trajectory. Therefore if we can show that the $t$'th term in the outer sum is in $O(T_t)$ we are done, so without loss of generality we consider the case of a single trajectory of length $n$ and show that the corresponding sum of covariances is in $O(n)$.

Now let $\alpha(k)$ denote the $k$th $\alpha$-mixing coefficient. Since $X_{1:n}$ is a Markov chain we have that $\alpha(X_i, X_j) = \alpha(\abs{i - j})$. In addition, given any random variables $U$ and $W$, it follows from \citet[1.12b]{rio2013inequalities} that $\covar[X,Y] \leq 2 \alpha(U,W) \norm{U}_{\infty} \norm{W}_{\infty}$. Applying this result to our setting we obtain 
\begin{align*}
  \covar[f(X_i), f(X_j)] &\leq 2 \alpha(\abs{i - j}) \norm{f(X)}_{\infty}^2 \\
  &\leq 4 \beta(\abs{i - j}) \norm{f(X)}_{\infty}^2,
\end{align*}
where the second inequality follows from the fact that $\beta$-mixing coefficients are larger than $\alpha$-mixing coefficients (up to a factor of $2$). Thus we can obtain the bound
\begin{align*}
\sum_{i=1}^n \sum_{j=1}^{i-1} \covar[f(X_i), f(X_j)] &\leq 4 \norm{f(X)}_{\infty}^2 \sum_{i=1}^n \sum_{j=0}^{i-1} \beta(j) \\
&\leq 4 n \norm{f(X)}_{\infty}^2 \sum_{j=1}^{\infty} \beta(j) \\
&\leq 4 n \norm{f(X)}_{\infty}^2 \sum_{j=1}^{\infty} j^{2/(p-2)} \beta(j) \\
&\leq O(n),
\end{align*}
where $2 < p \leq \infty$ is the constant referenced in \cref{asm:mixing}, and the final inequality follows from \cref{asm:mixing}.

Thus we have $\sum_{i=1}^n \sum_{j=1}^{i-1} \covar[f(X_i), f(X_j)] = O(n)$, which lets us conclude that $\e[V] = O(1/n)$.
\end{proof}

\begin{lemma}
\label{lem:minimax}
Let \cref{asm:compactness,asm:convexity,asm:continuity,asm:symmetry} be given. Then for every constant $M \geq 0$ we have
$$\inf_{W} \sup_{g \in \Gset} J_{\lambda}(W, g) \leq \sup_{g \in \Gset} \inf_{\norm{W} \leq M} B(W, g)^2 + \frac{\lambda}{n^2} M^2.$$
\end{lemma}

\begin{proof}[Proof of \cref{lem:minimax}]

By assumption \cref{asm:compactness,asm:continuity} $\mathcal{G}$ is compact, and $g \mapsto J_\lambda(W,g)$ is continuous for every $W$. This means that by the Extreme Value theorem we can replace the supremum over $\Gset$ with a maximum over $\Gset$ in the quantity we are bounding. Given this, we will proceed by bounding $\min_W \max_{g \in \mathcal{G}} B(W, \mu)$ using von Neumann's minimax theorem to swap the minimum and the maximum, and then use this to establish the overall bound for $J_\lambda(W,\mu)$. 

First, we can observe that $B(W, g)$ is linear, and therefore both convex and concave, in each of $W$ and $g$. Next, by \cref{asm:compactness,asm:convexity} $\mathcal{G}$ is convex and compact, and as argued already $g \mapsto B(W,g)$ is continuous for every $W$. In addition, $B(W, g)$ is also clearly continuous in $W$ for fixed $g$, and the set $\{W : \|W\| \leq M\}$ is obviously compact and convex for any non-negative $M$. Thus by von Neumann's minimax theorem we have the following for every $M \geq 0$:
\begin{equation}
\label{eq:minimax}
\min_{\|W\| \leq M} \max_{g \in \mathcal{G}} B(W, g) = \max_{\mu \in \mathcal{G}} \min_{\|W\| \leq M} B(W, g)
\end{equation}

Given this, we can bound $\min_W \max_{\mu \in \mathcal{F}} J(W, \mu)$ as follows, which is valid for any $M$:
\begin{align*}
\min_{W} \max_{g \in \mathcal{G}} J_\lambda(W, g) &= \min_{W} \max_{g \in \mathcal{G}} B(W, g)^2 + \frac{\lambda}{n^2} \norm{W}^2 \\
&\leq \min_{\|W\| \leq M} \max_{g \in \mathcal{G}} B(W, g)^2 + \frac{\lambda}{n^2}\|W\|^2 \\
&\leq \min_{\|W\| \leq M} \max_{g \in \mathcal{G}} B(W, g)^2 + \frac{\lambda}{n^2} M^2 \\
&= (\min_{\|W\| \leq M} \max_{g \in \mathcal{G}} \abs{B(W, g)})^2 + \frac{\lambda}{n^2}M^2 \\
&= (\max_{g \in \mathcal{G}} \min_{\|W\| \leq M} B(W, g))^2 + \frac{\lambda}{n^2}M^2 \\
&\leq (\max_{g \in \mathcal{G}} \min_{\|W\| \leq M} \abs{B(W, g)})^2 + \frac{\lambda}{n^2}M^2 \\
&= \max_{g \in \mathcal{G}} \min_{\|W\| \leq M} B(W, g)^2 + \frac{\lambda}{n^2}M^2 \\
\end{align*}

In these inequalities we use the fact that $\min_W \max_g B(W, g) = \min_W \max_g \abs{B(W,g)}$, which follows because $B(W,g) = -B(W,-g)$, and by \cref{asm:symmetry}, $g \in \Gset \iff -g \in \Gset$. In addition we use the fact that $x \mapsto x^2$ is a monotonic function on $\mathbb R^+$.

\end{proof}

\begin{lemma}
\label{lem:optimal-w}
Let some $g \in \Gset$ be given. Then as long as there exists $i \in [n]$ such that $h_g(Z_i) \neq 0$, there exists $W \in \mathbb R^n$ satisfying
$$
B(W,g) = 0
$$
and
$$
\norm{W}^2 = \frac{(\sum_{i=1}^n k_g(Z_i) h_g(Z_i))^2}{4\sum_{i=1}^n h_g(Z_i)^2} 
$$.
\end{lemma}

\begin{proof}[Proof of \cref{lem:optimal-w}]
We will prove this non-constructively by considering the value of the solution to the constrained optimization problem
\begin{align*}
  &\min_W \sum_{i=1}^n W_i^2 \\
  &\text{s.t. } \frac{1}{n} \sum_{i=1}^n h_g(Z_i) W_i - k_g(Z_i) = 0 \\
\end{align*}

The Lagrangian corresponding to this problem is
\begin{equation*}
  \mathcal L(W;\lambda) = \sum_{i=1}^n W_i^2 + \lambda(h_g(Z_i)W_i - k_g(Z_i))
\end{equation*}
It can easily be verified by taking derivatives that for fixed $\lambda$ this is minimized by setting $W_i = -\frac{1}{2} \lambda h_g(Z_i)$. Plugging in this $W$, we obtain the dual problem
\begin{align*}
  D &= \max_{\lambda \in \mathbb R} \sum_{i=1}^n -\frac{1}{4} h_g(Z_i)^2 \lambda^2 + \frac{k_g(Z_i) h_g(Z_i)}{2} \lambda \\
  &= \max_{\lambda \in \mathbb R} -\frac{1}{4} \left(\sum_{i=1}^n h_g(Z_i)^2 \right) \lambda^2 + \frac{1}{2} \left( \sum_{i=1}^n k_g(Z_i) h_g(Z_i) \right) \lambda
\end{align*}
Again taking derivatives, it is clear that this objective is maximized by
\begin{equation*}
  \lambda = \frac{\sum_{i=1}^n k_g(Z_i) h_g(Z_i)}{\sum_{i=1}^n h_g(Z_i)^2}.
\end{equation*}
Plugging in this solution we have that the maximum dual value objective is given by
\begin{equation*}
  D^* = \frac{(\sum_{i=1}^n k_g(Z_i) h_g(Z_i))^2}{4\sum_{i=1}^n h_g(Z_i)^2}
\end{equation*}

Finally we note that the original constrained optimization problem had only linear equality constraints, and under the assumption that $h_g(Z_i) \neq 0$ for some $i$ we can construct a feasible solution, so Slater's condition applies. Thus we can conclude that the minimum euclidean norm of $W$ satisfying $B(W,g) = 0$ is given by $D^*$, and therefore a $W$ satisfying our conditions must exist.

\end{proof}

\begin{lemma}
\label{lem:stochastic-eqcont}
Let \cref{asm:mixing,asm:bounded-g,asm:uniform-entropy} be given. Then we have
\begin{equation*}
  \sup_{g \in \Gset^*} \left| \frac{1}{n} \sum_{i=1}^n h(Z_i)^2 - \e[h(Z)^2] \right| = o_p(1),
\end{equation*}
where
\begin{equation*}
  h(Z) = \e[g_A(S,U) \mid Z].
\end{equation*}
\end{lemma}

\begin{proof}[Proof of \cref{lem:stochastic-eqcont}]

Let $2 < p \leq \infty$ be the fixed value of $p$ from \cref{asm:mixing}, and let $N(\epsilon) = \max_{a \in [n]} N_{[]}(\epsilon, \Gset^*_a, \lpspace{p})$. It follows easily from our assumptions that $\int_0^\infty \sqrt{\log N(\epsilon)} d \epsilon < \infty$.

Now, let $\Fset = \{f : f(z) = \e[g_A(S,U) \mid Z=z]\}$. Given any $f \in \Fset$ indexed by some $g = (g_1, \ldots, g_m) \in \Gset^*$, we let $(l_1,r_1),\ldots,(l_m,r_m)$ be $\epsilon/m$-brackets for $g_1,\ldots,g_m$ respectively in $\lpspace{p}$. Now clearly by linearity $(f_l, f_r)$ = $(\e[l_A(S,U) \mid Z=z], \e[r_A(S,U) \mid Z=z])$ is a bracket for $f$, and we have
\begin{align*}
  \e[\abs{f_l(Z) - f_r(Z)}^p]^{1/p} &= \e[\abs{\e[r_A(S,U) - l_A(S,U) \mid Z]}^p]^{1/p} \\
  &\leq \e[\e[\abs{r_A(S,U) - l_A(S,U)}^p \mid Z]]^{1/p} \\
  &= \e[\abs{r_A(S,U) - l_A(S,U)}^p ]^{1/p} \\
  &\leq \e\left[\sum_{a=1}^m \abs{r_a(S,U) - l_a(S,U)}^p \right]^{1/p} \\
  &\leq \sum_{a=1}^m \e\left[\abs{r_a(S,U) - l_a(S,U)}^p \right]^{1/p} \\
  &\leq \sum_{a=1}^m \frac{\epsilon}{m} \\
  &= \epsilon.
\end{align*}

Thus the $\lpspace{p}$-bracketing number for $\Fset$ must be at most $N(\epsilon / m)^m$, since we can ensure that every $f \in \Fset$ is in an $\epsilon$-bracket by constructing $\epsilon/m$-brackets for each class $\Gset^*_a$, and then contstructing an $\epsilon$-bracket for $\Fset$ from each possible combinatorial choice of selecting one $\Gset_a^*$ bracket for each $a \in [m]$ and combining these.

Next, consider the function class $\Fset^2 = \{f: f(z) = \tilde f(z)^2, \tilde f \in \Fset\}$. Now, given a bracket $(l,r)$ for $f \in \Fset$ we can construct a bracket $(l_2, r_2)$ for the corresponding element $f^2$ of $\Fset^2$, where
\begin{align*}
  l_2(z) &= \indicator{\text{sign}(l(z)) = \text{sign}(r(z))} \min(l(z)^2, r(z)^2) \\
  r_2(z) &= \max(l(z)^2, r(z)^2).
\end{align*}
In the case that $\indicator{\text{sign}(l(z)) = \text{sign}(r(z))}$ we have $r_2(z) - l_2(z) = (r(z) - l(z))(r(z) + l(z)) \leq C(r(z) - l(z))$ for some constant $C$, which follows because \cref{asm:bounded-g} implies that $\Fset$ must be uniformly bounded also. Also in the other case we have $r_2(z) - l_2(z) = r_2(z) \leq (r(z) - l(z))^2 \leq C (r(z) - l(z))$. Thus we have
\begin{align*}
  \e[\abs{r_2(Z) - l_2(Z)}^p]^{1/p} &\leq \e[C^p \abs{r(Z) - l(Z)}^p]^{1/p} \\
  &= C \e[\abs{r(Z) - l(Z)}^p]^{1/p}.
\end{align*}
Thus any $\epsilon/C$-bracketing of $\Fset$ gives a $\epsilon$-bracketing of $\Fset^2$, so the $\lpspace{p}$-bracketing number of $\Fset^2$ must be at most $N(\epsilon / (m C))^m$. Therefore we have that the function class $\Fset^2$ satisfies
\begin{align*}
  \int_0^\infty \sqrt{\log N_{[]}(\epsilon, \Fset^2, \lpspace{p})} d \epsilon &\leq \int_0^\infty \sqrt{m \log N(\epsilon / (mC))} d \epsilon \\
  &= m^{3/2} C \int_0^\infty \sqrt{\log N(\alpha)} d \alpha \\
  &< \infty.
\end{align*}

This finite uniform-entropy integral combined the $\beta$-mixing part of \cref{asm:mixing} implies that the stochastic process over $\Fset^2$ defined by
$$G_n(f) = \sqrt{n} \left( \frac{1}{n} \sum_{i=1}^n f(Z_i)^2 - \e[f(Z)^2] \right) $$
converges tightly to a limiting Gaussian process, by \citet[Theorem 11.24]{kosorok2007introduction}. Thus the stochastic process $G_n / \sqrt{n}$ converges tightly to the zero random variable, meaning that $\sup_{f \in \Fset^2} G_n(f) / \sqrt{n} = o_p(1)$. Finally we can observe that by construction
\begin{equation*}
  \sup_{g \in \Gset^*} \left| \frac{1}{n} \sum_{i=1}^n h(Z_i)^2 - \e[h(Z)^2] \right| = \sup_{f \in \Fset^2} G_n(f) / \sqrt{n},
\end{equation*}
which gives us our final result.

\end{proof}

\section{Omitted Proofs}

\begin{proof}[Proof of \cref{thm:mse-bound}]
We begin by providing a bound for the conditional MSE, $\e[(\hat\tau_W - \pv(\pi_e))^2 \mid Z_{1:n}]$. Define the sample average policy effect:
\begin{equation*}
  \sape(\pi_e) = \frac{1}{n} \sum_{i=1}^n \sum_{a=1}^m d(S_i) \pi_e(a \mid S_i, U_i) \mu_a(S_i, U_i).
\end{equation*}

We note that following the derivation in \cref{sec:alg} we have $\e[\sape(\pi_e)] = \pv(\pi_e)$. Given this and \cref{asm:ergodic,asm:mixing,asm:overlap,asm:bounded-mean}, it is clear that the conditions of \cref{lem:mixing} apply to $\e_b[(\sape(\pi_e) - \pv(\pi_e))^2]$, so this term must be $O(1/n)$. Thus by Markov's inequality and the law of total expectation we have $\e[(\sape(\pi_e) - \pv(\pi_e))^2 \mid Z_{1:n}] = O_p(1/n)$. Then, using the fact that $(x+y)^2 \leq 2x^2 + 2y^2$, we have
\begin{equation*}
  \e[(\hat\tau_W - \pv(\pi_e))^2 \mid Z_{1:n}] \leq 2\e[(\hat\tau_W - \sape(\pi_e))^2 \mid Z_{1:n}] + O_p(1/n).
\end{equation*}

Next,
we perform a bias variance decomposition of the RHS of this bound as follows: 
\begin{align*}
  \e[(\hat\tau_W - \sape(\pi_e))^2 \mid Z_{1:n}] &= \e[ \e[(\hat\tau_W - \sape(\pi_e))^2 \mid Z_{1:n}, U_{1:n}] \mid Z_{1:n}] \\
  &= \e[ \e[\hat\tau_W - \sape(\pi_e) \mid Z_{1:n}, U_{1:n}]^2 \mid Z_{1:n}] \\
  &\qquad + \e[ \var[\hat\tau_W - \sape(\pi_e) \mid Z_{1:n}, U_{1:n}] \mid Z_{1:n} ] \\
  &= \xi_1 + \xi_2,
\end{align*}
and we additionally define
\begin{align*}
    \zeta_{ia} &= W_i \delta_{A_i a} R_i - d(S_i) \pi_e(a \mid S_i, U_i) \mu_a(S_i, U_i) \\
    \zeta_i &= \sum_{a=1}^m \zeta_{ia} = W_i R_i - d(S_i) \sum_{a=1}^m \pi_e(a \mid S_i, U_i) \mu_a(S_i, U_i).
\end{align*}

We note that our MDPUC structure implies that $R_i$ and $U_i$ are conditionally independent of all other states, actions, rewards, and confounders given $Z_i$, and therefore that $\e[\zeta_{ia} \mid Z_{1:n}] = \e[f_{ia} \mu_a(S_i, U_i) \mid Z_i]$. Given this, the first term of the above bias variance decomposition can be broken down as:
\begin{align*}
  \xi_1 &= \e \left[ \left( \frac{1}{n} \sum_{i=1}^n \sum_{a=1}^m \zeta_{ia} \right)^2 \mathrel{\Big|} Z_{1:n} \right] \\
  &= \e \left[ \frac{1}{n} \sum_{i=1}^n \sum_{a=1}^m \zeta_{ia} \mathrel{\Big|} Z_{1:n} \right]^2 + \var \left[ \frac{1}{n} \sum_{i=1}^n \sum_{a=1}^m \zeta_{ia} \mathrel{\Big|} Z_{1:n} \right] \\
  &= \left( \frac{1}{n} \sum_{i=1}^n \sum_{a=1}^m \e[ f_{ia} \mu_a(S_i,U_i) \mid Z_i] \right)^2 + \var\left[ \frac{1}{n} \sum_{i=1}^n \zeta_i \mid Z_{1:n} \right] \\
  &\leq \left( \frac{1}{n} \sum_{i=1}^n \sum_{a=1}^m \e[ f_{ia} \mu_a(S_i,U_i) \mid Z_i] \right)^2 + \frac{2\sigma^2}{n^2} \sum_{i=1}^n W_i^2 \\
  &\qquad + 2 \var\left[ \frac{1}{n} \sum_{1=1}^n d(S_i) \sum_{a=1}^m \pi_e(a \mid S_i, U_i) \mu_a(S_i, U_i) \mid Z_{1:n} \right] \\
  &= B(W,\mu)^2 + \frac{2\sigma^2}{n^2} \norm{W}^2 + 2 \var\left[ \frac{1}{n} \sum_{1=1}^n d(S_i) \sum_{a=1}^m \pi_e(a \mid S_i, U_i) \mu_a(S_i, U_i) \mid Z_{1:n} \right],
\end{align*}
where the inequality step follows from \cref{asm:bounded-variance} and the identity $(x+y)^2 \leq 2x^2 + 2y^2$. Similarly, we bound the the second error term $\xi_2$ as:
\begin{align*}
  \xi_2 &= \e \left[ \var \left[ \frac{1}{n} \sum_{i=1}^n \zeta_i \mathrel{\Big|} Z_{1:n}, U_{1:n} \right] \mathrel{\Big|} Z_{1:n} \right] \\
  &\leq \e \left[ \frac{2\sigma^2}{n^2} \sum_{i=1}^n W_i^2 + 2 \var \left[  \frac{1}{n} \sum_{i=1}^n d(S_i) \sum_{a=1}^m \pi_e(a \mid S_i,U_i) \mu_a(S_i,U_i) \mathrel{\Big|} Z_{1:n}, U_{1:n} \right] \mathrel{\Big|} Z_{1:n} \right] \\
  &\leq \frac{2\sigma^2}{n^2} \norm{W}^2 + 2 \var\left[ \frac{1}{n} \sum_{1=1}^n d(S_i) \sum_{a=1}^m \pi_e(a \mid S_i, U_i) \mu_a(S_i, U_i) \mathrel{\Big|} Z_{1:n} \right],
\end{align*}
where in the first inequality step follows again from \cref{asm:bounded-variance} and the identity $(x+y)^2 \leq 2x^2 + 2y^2$, and the second inequality step follows from the law of total variance.

Next, by \cref{asm:ergodic,asm:mixing,asm:overlap,asm:bounded-mean}, it follows from \cref{lem:mixing} that 
\begin{equation*}
    \var\left[ \frac{1}{n} \sum_{1=1}^n d(S_i) \sum_{a=1}^m \pi_e(a \mid S_i, U_i) \mu_a(S_i, U_i) \right] = O(1/n),
\end{equation*}
and therefore it follows from Markov's inequality that the conditional variance version is $O_p(1/n)$.

Next, putting the above bounds together we get
\begin{equation*}
  \e[(\hat\tau_W - \pv(\pi_e))^2 \mid Z_{1:n})] \leq 2\left( B(W,\mu)^2 + \frac{4 \sigma^2}{n^2} \norm{W}^2 \right) + O_p(1/n).
\end{equation*}

It follows from this that if $\lambda \geq 4\sigma^2$ and $J_\lambda(W,\mu) = O_p(r_n)$, then $\e[(\hat\tau_W - \pv(\pi_e))^2 \mid Z_{1:n}] = O_p(\max(1/n, r_n))$. Finally, it follows from \citet[Lemma 31]{kallus2016generalized} that $(\hat\tau_W - \pv(\pi_e))^2 = O_p(\max(1/n, r_n))$, and thus $\hat\tau_W = \pv(\pi_e) + O_p(\max(n^{-1/2}, r_n^{1/2}))$.

\end{proof}

\begin{proof}[Proof of \cref{thm:bound-convergence}]
We first note that by \cref{lem:minimax} we have for every $M \geq 0$:
$$\inf_{W \in \mathbb R^n} \sup_{g \in \Gset} J_\lambda(W,g) \leq \sup_{g \in \Gset} \inf_{\norm{W} \leq M} B(W,g)^2 + \frac{\lambda}{n^2} M^2.$$

Therefore it is sufficient to ensure that, for each $g \in \Gset$, that we can find $W(g)$ in response such that $B(W(g), g) = 0$ and $\sup_{g \in \Gset} \norm{W(g)}^2 = O_p(n)$. In the case that $\norm{g} = 0$ we have from \cref{asm:continuity} that $B(0,g)=0$, so we can easily restrict our attention in proving this to $g$ with norm greater than zero.

Now given the decomposition $B(W,g) = \frac{1}{n} \sum_{i=1}^n W_i h_g(Z_i) - k_g(Z_i) $, as long as $h(Z_i) \neq 0$ for some $i \in [n]$ it follows from \cref{lem:optimal-w} that we can find $W(g)$ satisfying
\begin{align*}
  B(W(g), g) &= 0 \\
  \norm{W(g)}^2 &= \frac{(\sum_{i=1}^n h(Z_i) k(Z_i))^2}{4 \sum_{i=1}^n h(Z_i)^2} \\
  &= n \frac{(\frac{1}{n}\sum_{i=1}^n h(Z_i) k(Z_i))^2}{4 \frac{1}{n}\sum_{i=1}^n h(Z_i)^2} \\
  &= n \frac{(\frac{1}{n}\sum_{i=1}^n h(Z_i) k(Z_i))^2}{4(\e[h(Z)^2] + (\frac{1}{n}\sum_{i=1}^n h(Z_i)^2 - \e[h(Z)^2]))}.
\end{align*}

We note that this equation clearly satisfies $\norm{W(g)}^2 = \norm{W(\lambda g)}^2$ for any $\norm{g} \neq 0$ and $\lambda > 0$. Thus it follows that $\sup_{g \in \Gset} \norm{W(g)}^2 = \sup_{g \in \Gset^*} \norm{W(g)}$. Furthermore, by \cref{asm:non-degenerate} we have that $P(h_g(Z) > 0)$ for every $g \in \Gset^*$, and thus $\e[h_g(Z)^2] > 0$. Now let $\alpha = \inf_{g \in \Gset^*} \e[h(Z)^2]$. Given \cref{asm:compactness,asm:continuity} the extreme value theorem applies and we have $\alpha > 0$. Next, it follows easily from the uniform boundedness of \cref{asm:bounded-g} that $(\frac{1}{n}\sum_{i=1}^n h(Z_i) k(Z_i))^2 \leq \beta(\frac{1}{n} \sum_{i=1}^n d(S_i))$ for some $0 < \beta < \infty$. Furthermore \cref{asm:ergodic} gives us that $d(S_i)$ is stationary, it follows from the Markov chain law of large numbers that $\beta(\frac{1}{n} \sum_{i=1}^n d(S_i)) = O_p(1)$. Thus for any $g \in \Gset^*$ we have
\begin{equation*}
  \norm{W(g)}^2 \leq \left( \frac{n}{4} \right) \frac{O_p(1)}{\alpha + \epsilon(g) + \frac{1}{n}\sum_{i=1}^n h(Z_i)^2 - \e[h(Z)^2]},
\end{equation*}
where $\alpha > 0$, and $\epsilon(g) \geq 0$. Next, \cref{lem:stochastic-eqcont} implies that the stochastic equicontinuity term $(1/n) \sum_{i=1}^n h(Z_i)^2 - \e[h(Z)^2]$ converges in probability to 0 uniformly over $\Gset$. Thus by the continuous mapping theorem we have that the RHS of the previous bound converges in probability to $O_p(n) / (\alpha + \epsilon(g)) \leq O_p(n)$ uniformly over $g \in \Gset^*$.

Now recall that this bound was valid in the event that at least one $h(Z_i)$ is non-zero, which must occur with probability $1 - \delta_n(g)$, where $\delta_n(g) = O_p(p(g)^{-n})$, and $p(g) = P(h(Z) = 0)$. Furthermore given \cref{asm:compactness,asm:continuity,asm:non-degenerate} and applying the extreme value theorem as above, we have $\sup_{g \in \Gset^*} \delta_n(g) = O_p(p^{-n})$, for some $p < 1$. In the event that for some $g$ every $h(Z_i)$ is zero, we can instead choose $W(g) = 0$, giving a bound of $J_\lambda(W(g), g) \leq (\sum_{i=1}^n k(Z_i))^2 = O_p(1)$ uniformly over $g \in \Gset^*$, since $\frac{1}{n} \sum_{i=1}^n k(Z_i)$ can be=  bounded uniformly over $\Gset^*$ by applying \cref{asm:ergodic,asm:bounded-g} as discussed above.

Therefore we can conclude by putting the above bounds together, which gives us
\begin{equation*}
\inf_{W \in \mathbb R^n} \sup_{g \in \Gset} J_\lambda(W, g) \leq (1 - O_p(p^{-n})) O_p(1 / n) + O_p(p^{-n}) O_p(1) = O_p(1 / n).
\end{equation*}

\end{proof}

\begin{proof}[Proof of \cref{lem:estimation-error}]

Recall that for this lemma we have made the assumption that $\pi_e$ is measurable with respect to $S$ only. That is, $\pi_e(a \mid s,u) = \pi_e(a \mid s)$. Let $b$ be some constant such that $g_a(s,u) \leq b$ for every $g \in \mathcal G$, $a \in [m]$, $s \in \mathcal S$, and $u \in \mathcal U$, and let $c$ be some constant such that $d(s) \leq c$ for every $s \in \mathcal S$. We note that both these constants must exist given \cref{asm:overlap,asm:bounded-g}. In addition we define the estimated versions of the quantities in our analysis as follows.
\begin{align*}
    \hat f_{ia} &= W_i \delta{A_i a} - \hat d(S_i) \pi_e(a \mid S_i) \\
    \hat B(W,g) &= \frac{1}{n} \sum_{i=1}^n \sum_{a=1}^m \hat f_{ia} \e_{\hat \varphi}[ g_a(S_i, U_i) \mid Z_i ] \\
    \hat J_{\lambda}(W,g) &= \hat B(W,g)^2 + \frac{\lambda}{n^2} \norm{W}^2.
\end{align*}

Given this, for any $W$ measurable in $Z_{1:n}$, we can obtain the bound
\begin{align*}
    \abs{\sup_{g \in \mathcal G} &B(W,g) - \sup_{g \in \mathcal G} \hat B(W,g)} \\
    &\leq \sup_{g \in \mathcal G} \abs{B(W, g) - \hat B(W, g)} \\
    &\leq \sup_{g \in \mathcal G} \left| \frac{1}{n} \sum_{i=1}^n \sum_{a=1}^m f_{ia} (\e - \hat\e)[g_a(S_i,U_i) \mid Z_i] \right| \\
    &\qquad + \sup_{g \in \mathcal G} \left| \frac{1}{n} \sum_{i=1}^n (d(S_i) - \hat d(S_i)) \hat \e\left[ \sum_{a=1}^m \pi_e(a \mid S_i) g_a(S_i,U_i) \mathrel{\Big|} Z_i \right] \right| \\
    &\leq \frac{F}{n} \sum_{i=1}^n \sum_{a=1}^m \abs{f_{ia}} D_{\mathcal F}(\varphi(Z_i), \hat \varphi(Z_i)) + \frac{b}{n} \sum_{i=1}^n \abs{d(S_i) - \hat d(S_i)} \\
    &\leq \frac{F}{n} \sum_{i=1}^n (\abs{W_i} + c) D_{\mathcal F}(\varphi(Z_i), \hat \varphi(Z_i)) + \frac{b}{n} \sum_{i=1}^n \abs{d(S_i) - \hat d(S_i)} \\
    &\leq \frac{cF}{n} \sum_{i=1}^n D_{\mathcal F}(\varphi(Z_i), \hat \varphi(Z_i)) + \frac{F \norm{W}}{\sqrt{n}} \left( \frac{1}{n} \sum_{i=1}^n D_{\mathcal F}(\varphi(Z_i), \hat\varphi(Z_i))^2 \right)^{1/2} \\
    &\qquad + \frac{b}{n} \sum_{i=1}^n \abs{d(S_i) - \hat d(S_i)} \\
    &\leq O_p(r_n) + \frac{\norm{W}}{\sqrt{n}} O_p(r_n) + O_p(r_n),
\end{align*}

where in the second last inequality we apply Cauchy Schwartz, and in the final inequality we apply the assumptions that $D_\mathcal F(\varphi(Z_i), \hat\varphi(Z_i)) = O_p(r_n)$ and $\abs{d(S_i) - \hat d(S_i)} = O_p(r_n)$ for every $i \in [n]$. Now, let $\tilde W = \argmin_W \sup_{g \in \mathcal G} J_{\lambda}(W, g)$. It easily follows from \cref{thm:bound-convergence} that $\norm{\tilde W} = O_p(\sqrt{n})$, so from the above we have $\sup_{g \in \mathcal G} \hat B(\tilde W, g) = \sup_{g \in \mathcal G} B(\tilde W, g) + O_p(r_n)$. In addition it also follows from \cref{thm:bound-convergence} that $\sup_{g \in \mathcal G} B(\tilde W, g) = O_p(1/\sqrt{n})$. Putting all of the above together we get $\sup_{g \in \mathcal G} \hat J_{\lambda}(\tilde W,g) = O_p(\max(1/n, r_n^2))$, and therefore $\sup_{g \in \mathcal G} \hat J_{\lambda}(W^*,g) = O_p(\max(1/n, r_n^2))$. Given this, it follows that $\norm{W^*} = O_p(\max(\sqrt{n}, n r_n))$, and therefore applying the bound above again we get
\begin{align*}
    &\sup_{g \in \mathcal G} B(W^*,g) = O_p(\max(\sqrt{n} r_n^2, r_n)) \\
    &\implies \sup_{g \in \mathcal G} J_{\lambda}(W^*,g) = O_p(\max(1/n, r_n^2, n r_n^4)) = O_p(\max(1/n, n r_n^4)),
\end{align*}
where the final equality follows since it is always the case that either $1 \leq n r_n^2 \leq n^2 r_n^4$, or $1 \geq n r_n^2 \geq n^2 r_n^4$ (depending on whether $n r_n^2 \geq 1$ or not). It immediately follows that $J_{\lambda}(W^*, \mu) = O_p(\max(1/n, n r_n^4))$. Therefore plugging in $\lambda = 4 \sigma^2$, the required result immediately follows by applying \cref{thm:mse-bound}.

\end{proof}

\begin{proof}[Proof of \cref{thm:non-iid-bound}]

First, following exactly the same argument as in the proof of \cref{lem:estimation-error}, we can obtain the bound
\begin{equation*}
    \sup_{g \in \mathcal G} J_{\lambda}(W^*, g) = O_p(\max(1/n, n r_n^4)).
\end{equation*}

We note that none of the arguments or theorems used in the derivation of the above bound, or \cref{thm:bound-convergence} which is used in the argument, depend on the assumption that the confounders are independent, and therefore this bound still holds in the case that $U_{1:n}$ are distributed according to a Markov chain.

Next, we define the following terms similar to those in our core theory, recalling that for this lemma we have assumed that $pi_e$ is measurable with respect to the observed state only (that is, $\pi_e(s,u) = \pi_e(s)$).
\begin{align*}
    f^*_{ia} &= W_i^* \delta_{A_i a} - d(S_i, U_i) \pi_e(a \mid S_i) \\
    B^*(W^*,g) &= \frac{1}{n} \sum_{i=1}^n \sum_{a=1}^m f^*_{ia} \e[g_a(S_i, U_i) \mid Z_{1:n}].
\end{align*}

In addition, we define the error term
\begin{equation*}
    \epsilon(W^*,\mu) = \abs{B^*(W^*,\mu)^2 - B(W^*,\mu)^2}^{1/2}.
\end{equation*}

Then it follows from \cref{lem:non-iid-cmse-bound} (described and proved in \cref{apx:sensitivity-theory}) that
\begin{equation*}
    \e[(\hat\tau_{W^*} - v(\pi_e))^2 \mid Z_{1:n}] \leq 2 J_{4 \sigma^2}(W^*,\mu) + \epsilon(W^*,\mu)^2.
\end{equation*}
Given this, the assumption that $\lambda \geq 4 \sigma^2$, and the bound $\sup_{g \in \mathcal G} J_{\lambda}(W^*,g) = O_p(\max(1/n, n r_n^4))$, it follows from \citet[Lemma 31]{kallus2016generalized} that 
\begin{equation*}
    (\hat\tau_{W^*} - \pv(\pi_e))^2 = \epsilon(W^*, \mu)^2 + O_p(\max(1/n, n r_n^4))
\end{equation*}
That is, by bounding $\epsilon(W^*,\mu)$ we can bound the irreducible MSE from our balanced policy evaluation in the non-iid setting.

Next, we let $b$ be a constant such that $\abs{\mu_a(s,u)} \leq b \ \forall a,s,u$ (which must exist given \cref{asm:bounded-mean}).
Given our assumption that $\pi_e$ is measurable with respect to $S$, it follows from \cref{asm:f-norm} that the assumptions of \cref{lem:non-iid-bias-bound} are satisfied (described and proved in \cref{apx:sensitivity-theory}). Then applying the fact that $J_{\lambda}(W^*,g) = O_p(\max(1/n,nr_n^4))$ implies that $\abs{B(W^*,\mu)} = O_p(\max(n^{-1/2}, n^{1/2} r_n^2))$, as well as Cauchy Schwartz and the inequality $(x+y)^2 \leq 2x^2 + 2y^2$, this lemma gives us

\begin{equation*}
    \epsilon(W^*,\mu) \leq F c \left( \frac{1}{n} \sum_{i=1}^n D_{\mathcal F}(\varphi_{1:n}, \varphi_i)^2 \right)^{1/2} + b \norm{d(S,U) - d(S)}_2 + O_p(\max(n^{-1/2}, n^{1/2} r_n^2)),
\end{equation*}
where $c = \sqrt{2} (\norm{W^*}^2/n + 1)^{1/2}$, which gives us our final result.



\end{proof}

\begin{proof}[Proof of \cref{thm:stationary-density-ratio}]

We first note that $\e[d(S)] = 1$ follows trivially for any stationary density ratio, by the definition of $d$ and the fact that all probability measures have total measure 1.

Next, let $U'$ denote the successor of $U$ (analogously to $S'$), and let $f_b$ and $f_e$ refer to measures (or conditional measures) with respect to the stationary distributions of $\pi_b$ and $\pi_e$. Then we have
\begin{align*}
  \e_b[d(S,U,A) \mid S', U'] &= \int d(s,u,a) f_b(s,u,a \mid S',U') ds \ du \ da \\
  &= \int \frac{f_e(s,u,a)}{f_b(s,u,a)} f_b(s,u,a \mid S', U') ds \ du \ da \\
  &= \int \frac{f_e(s,u,a)}{f_b(s,u,a)} \frac{f_b(S',U' \mid s,u,a) f_b(s,u,a)}{f_b(S',U')} ds \ du \ da \\
  &= \int \frac{f_e(s,u,a) f_e(S',U' \mid s,u,a)}{f_b(S',U')} ds \ du \ da \\
  &= d(S',U').
\end{align*}
In the second last step we appeal to the fact that the conditional density of $S',U'$ given $S,U,A$ is the same under both $\pi_b$ and $\pi_e$ given our MDPUC assumptions. Note also that the fractions in the above derivation should be interpreted as Radon–Nikodym derivatives where appropriate, in the case that the random variables are not continuous.

Next, we note that $d(S,U,A) = d(S,U) \pi_e(A \mid S,U) / \pi_b(A \mid S,U)$, and that by our MDPUC assumtions we have that $d(S,U) = d(S)$. Therefore we have
\begin{align*}
  \e_b \left[ d(S) \frac{\pi_e(A \mid S,U)}{\pi_b(A \mid S,U)} - d(S') \mathrel{\Big|} S', U' \right] = 0.
\end{align*}
Next we note that from our MDPUC indepdence assumptions $(S,A,U)$ are indepdendent of $U'$ given $S'$, so we can marginalize over $U'$ and obtain
\begin{align*}
  \e_b \left[ d(S) \frac{\pi_e(A \mid S,U)}{\pi_b(A \mid S,U)} - d(S') \mathrel{\Big|} S' \right] = 0.
\end{align*}
Finally we can iterate expectations on $Z$ to obtain
\begin{align*}
  \e_b \left[ d(S) \beta(Z) - d(S') \mathrel{\Big|} S' \right] = 0.
\end{align*}

Now we have established that the true stationary density ratio must satisfy the regular and conditional moment conditions described in \cref{thm:stationary-density-ratio}. For the reverse result, we note first that \cref{asm:ergodic} implies that the stationary distribution of our Markov chain is unique. Now as argued in \citet{liu2018breaking}, it is clear given ergodicity that any $d$ satisfying this conditional moment restriction must correspond to a scalar multiple of the true stationary density ratio, since the construction of the conditional moment restriction is exactly identical to that of \citet{liu2018breaking} if we consider $(S,U)$ to be the state. Thus the additional restriction that $\e[d(S)] = 1$ ensures that any $d$ satisfying both moment conditions must the true stationary density ratio.

\end{proof}

\begin{proof}[Proof of \cref{lem:kernel-g}]

First we observe that by construction $\Gset_K = \Gset_K^*$, so we will only discuss the former. Define
\begin{align*}
  B_s &= \sup_{s \in \Sset, u \in \Uset} \sqrt{K((s,u),(s,u))}.
\end{align*}
By our bounded kernel assumption we have that $0 < B_s < \infty$. Now, for any $g \in \Gset_a^*$ we have $g((s,u)) = \inner{g}{K_{s,u}} \leq B_s \norm{g}$, where $K_{s,u}$ denotes the reproducing element for evaluation at $s,u$. Thus $\norm{g}_{\infty} \leq B_2 \norm{g}$, which gives us \cref{asm:bounded-g}. 

Next, from \citet[Theorem D]{cucker2002mathematical} we have that the covering number under the $\lpspace{\infty}$-norm of an RKHS ball of unit radius with bounded, continuous kernel is given by
$$\sqrt{\log N(\epsilon, \Gset^*_a, \lpspace{\infty})} \leq (C_b / \epsilon)^b,$$
for some constant $C_b > 0$ depending only on $b$ and any $0 < b < 1$. Thus it is easy to argue by constructing separate finite covering sets for each $a \in [m]$ that we satisfy \cref{asm:compactness}.

In order to deal with \cref{asm:uniform-entropy} we note that an $\lpspace{\infty}$ covering number bound gives a corresponding $\lpspace{\infty}$ bracketing number bound given \cref{asm:bounded-g}. Concretely, given any $g \in \Gset_a^*$, we let $g'$ be a function such that $\norm{g - g'}_\infty < \epsilon$. This implies that the bracket $(g'-\epsilon, g'+\epsilon)$ is a valid bracket for $g$. Therefore $N_{[]}(\epsilon, \Gset_a^*, \lpspace{\infty}) \leq N(\epsilon, \Gset_a^*, \lpspace{\infty})$. Thus we have that:
\begin{equation*}
  \sqrt{\log N_{[]}(\epsilon, \Gset_a^*, \mathcal L_{\infty})} \leq (C_b / \epsilon)^b,
\end{equation*}
which is sufficient to ensure \cref{asm:uniform-entropy} since $\int_0^C (1/\epsilon)^b d \epsilon < \infty$ for any $0<C<\infty$ when $0 < b < 1$, and from \cref{asm:bounded-g} we have that $\sqrt{\log N_{[]}(\epsilon, \Gset_a^*, \mathcal L_{\infty})} = 0$ when $\epsilon \geq B_s$.

Finally we note that \cref{asm:symmetry,asm:convexity} are trivial from the definition of $\Gset_K$, as is \cref{asm:continuity} since RKHSs are continuous with respect to function application.

\end{proof}

\begin{proof}[Proof of \cref{thm:w-objective}]

First we will find a closed form expression for
\[
\sup_{g \in \Gset_K} (\frac{1}{n} \sum_{i=1}^n \sum_{a=1}^m \e[f_{ia} g_(S_i, U_i) \mid Z_i])^2.
\]
In this derivation we will use the shorthand $\varphi_i$ for the conditional density of $U_i$ given $Z_i$, and $T_K$ for the kernel intergral operator defined according to $T_K f = \int_{\mathcal Z} K(\cdot,z) f(z) dz$. In this derivation we will make use of the fact that $\inner{f}{g}_2 = \inner{f}{g}_K$ for any square integrable $f$ and $g$, where these inner products refer to $\lpspace{2}$ and the RKHS $\Hset_K$ respectively. Note that in this derivation we calculate $\lpspace{2}$ inner products with respect to the Borel measure $\mathbb R$, rather than the measure from the stationary distribution of $\pi_b$, which allows us to write conditional expectations as explicit inner products. Given all this we can obtain:
\begin{align*}
\sup_{g \in \Gset} &\left( \frac{1}{n} \sum_{i=1}^n \sum_{a=1}^m \e \left[ f_{ia} g_a(S_i, U_i) \mathrel{\Big|} Z_i \right] \right)^2 \\
&= \sum_{a=1}^m \sup_{g \in \Gset} \left( \frac{1}{n} \sum_{i=1}^n \inner{g_a}{\varphi_i f_{ia}}^2_2 \right)^2 \\
&= \sum_{a=1}^m \sup_{g \in \Gset} \left( \inner{g_a}{T_K \frac{1}{n} \sum_{i=1}^n \varphi_i f_{ia}}_K^2 \right)^2 \\
&= \sum_{a=1}^m \frac{\inner{T_k \frac{1}{n} \sum_{i=1}^n \varphi_i f_{ia}}{T_k \frac{1}{n} \sum_{i=1}^n \varphi_i f_{ia}}_K^2}{\norm{T_k \frac{1}{n} \sum_{i=1}^n \varphi_i f_{ia}}_K} \\
&= \sum_{a=1}^m \inner{T_k \frac{1}{n} \sum_{i=1}^n \varphi_i f_{ia}}{T_k \frac{1}{n} \sum_{i=1}^n \varphi_i f_{ia}}_K \\
&= \sum_{a=1}^m \inner{\frac{1}{n} \sum_{i=1}^n \varphi_i f_{ia}}{T_k \frac{1}{n} \sum_{i=1}^n \varphi_i f_{ia}}_2 \\
&= \frac{1}{n^2} \sum_{i,j=1}^n \sum_{a=1}^m \int \varphi_i(u) f_{ia} \left( \int K((S_i,u),(S_j,u')) \varphi_j(u') f_{ja} \ du' \right) du \\
&= \frac{1}{n^2} \sum_{i,j=1}^n \sum_{a=1}^m \int \int \varphi_i(u) f_{ia} \varphi_j(u') f_{ja} K((S_i,u),(S_j,u')) \ du \ du' \\
&= \frac{1}{n^2} \sum_{i,j=1}^n \sum_{a=1}^m \e[f_{ia} \tilde f_{ja} K(((S_i,U_i), (S_j, \tilde U_j)) \mid Z_i, Z_j]
\end{align*}

Next, we convert this into a quadratic objective in $W$. Recall that $f_{ia} = W_i \delta_{A_i a} - d(S_i) \pi_e(a \mid S_i, U_i)$, and $k_{ij} = K((S_i,U_i),(S_j,\tilde U_j))$. Then given this immediately follows from basic matrix algebra that
\begin{align*}
  \sup_{g \in \Gset_K} J_{\lambda}(W,g) &= \sup_{g \in \Gset_K} B(W,g) + \frac{\lambda}{n^2} \sum_{i=1}^n W_i^2 \\
  &= \frac{1}{n^2} \sum_{i,j=1}^n \sum_{a=1}^m \e[f_{ia} \tilde f_{ja} k_{ij} \mid Z_i, Z_j] + \frac{\lambda}{n^2} \sum_{i=1}^n W_i^2 \\
  &= \frac{1}{n^2} \sum_{i,j=1}^n W_i W_j \left( \delta_{A_i A_j} \e[k_{ij} \mid Z_i, Z_j] + \lambda \delta_{ij} \right) \\
  &\qquad -2 \frac{1}{n^2} \sum_{i,j=1}^n W_i d(S_j) \e[ \pi_e(A_i \mid S_j, U_j) k_{ij} \mid Z_i, Z_j ] \\
  &\qquad + \frac{1}{n^2} \sum_{i,j=1}^n d(S_i) d(S_j) \e[ \sum_{a=1}^m (\pi_e(a \mid S_i, U_i) \pi_e(a \mid S_j, U_j)) k_{ij} \mid Z_i, Z_j ]
\end{align*}

Thus we have $\sup_{g \in \Gset_K} J_{\lambda}(W,g) = W^T G W - 2 g + C$, where $G$ and $g$ are defined as in the proof statement, and
\begin{equation*}
  C = \frac{1}{n^2} \sum_{i,j=1}^n d(S_i) d(S_j) \e[ \sum_{a=1}^m (\pi_e(a \mid S_i, U_i) \pi_e(a \mid S_j, U_j)) k_{ij} \mid Z_i, Z_j ],
\end{equation*}
which is clearly indepdendent of $W$.

\end{proof}

\section{Sensitivity Theory}
\label{apx:sensitivity-theory}

In this appendix we present some details on the sensitivity of our theory under minor violations of the iid confounders assumption. We consider a generalization of the MDPUC model depicted in \cref{fig:mdpuc}, where we allow the unobserved confounder values to be correlated rather than assuming them to be iid. For this analysis we define the following terms similar to those in our core theory:
\begin{align*}
    f^*_{ia} &= W_i \delta_{A_i a} - d(S_i, U_i) \pi_e(a \mid S_i, U_i) \\
    B^*(W,g) &= \frac{1}{n} \sum_{i=1}^n \sum_{a=1}^m \e[f^*_{ia} g_a(S_i, U_i) \mid Z_{1:n}] \\
    J^*_{\lambda}(W,g) &= B^*(W,g)^2 + \frac{\lambda}{n^2} \norm{W}^2.
\end{align*}

We note that these only differ from the original terms in two respects: (1) conditioning on all observed triplets $Z_{1:n}$ rather than the single observed triplet $Z_i$ in the $i$'th term; and (2) use of density ratio $d(S,U)$ rather than $d(S)$. Given this, we can first obtain the following lemma under a mild modification of our overlap assumption.

\begin{assumption}
\label{asm:overlap-non-iid}
$\norm{d(S,U)}_q < \infty$, where $2 < q \leq \infty$ is the same value referred to in \cref{asm:mixing}.
\end{assumption}

\begin{lemma}
\label{lem:non-iid-cmse-bound}
Let \cref{asm:ergodic,asm:mixing,asm:bounded-mean,asm:bounded-variance,asm:overlap-non-iid} be given. Then we have
\begin{equation*}
    \e[(\hat\tau_W - \pv(\pi_e))^2 \mid Z_{1:n}] \leq 2 J^*_{4\sigma^2}(W, \mu) + O_p(1/n).
\end{equation*}
\end{lemma}

This proof of this lemma is almost identical to that of \cref{thm:mse-bound}, and is detailed in \cref{apx:sensitivity-proofs}. Next, we define the error term
\begin{equation*}
    \epsilon(W,\mu) = \abs{B^*(W,\mu)^2 - B(W,\mu)^2}^{1/2}.
\end{equation*}

Then it follows that
\begin{equation*}
    \e[(\hat\tau_W - v(\pi_e))^2 \mid Z_{1:n}] \leq 2 J_{4 \sigma^2}(W,\mu) + \epsilon(W,\mu)^2,
\end{equation*}
and therefore if we choose $W$ such that $J_{4 \sigma^2}(W,\mu) = O_p(r_n)$,  then applying \citet[Lemma 31]{kallus2016generalized} gives us
\begin{equation*}
    (\hat\tau_W - \pv(\pi_e))^2 = \epsilon(W,\mu)^2 + O_p(\max(1/n, r_n)).
\end{equation*}
That is, by bounding $\epsilon(W,\mu)$ we can bound the irreducible bias from our balanced policy evaluation in the non-iid setting. Note that our theory for providing conditions where $J_{4 \sigma^2}(W,\mu) = O_p(1/n)$ (from \cref{thm:bound-convergence}, assuming no nuisance error), or $J_{4 \sigma^2}(W,\mu) = O_p(\max(1/n, n r_n^4))$ (assuming $O_p(r_n)$ nuisance error) does not depend on the assumption that $U_i$ values are iid, and therefore still applies here.


Next, we let $b$ be a constant such that $\abs{\mu_a(s,u)} \leq b \ \forall a,s,u$ (which must exist given \cref{asm:bounded-mean}), and we let $D_{\mathcal F}$ be defined as in \cref{sec:realistic-setting}, and we let $\varphi_i$ and $\varphi_i^*$ be defined as in \cref{thm:non-iid-bound}. Given these definitions, we provide the following result on the residual bias $\epsilon(W,\mu)$:

\begin{lemma}
\label{lem:non-iid-bias-bound}
Suppose $F$ is some constant such that for every $S \in \mathcal S, A \in [m]$ we have $\norm{\mu_A(S,\cdot)}_{\mathcal F} \leq F$ and $\norm{\sum_{a=1}^m \pi_e(a \mid S,\cdot) \mu_a(S,\cdot)}_{\mathcal F} \leq F$. Then given \cref{asm:ergodic,asm:mixing,asm:bounded-mean,asm:bounded-variance,asm:overlap-non-iid}, we have 
\begin{equation*}
    \epsilon(W,\mu) \leq F \left( \frac{1}{n} \sum_{i=1}^n (\abs{W_i}+1) D_{\mathcal F}(q_{1:n}, q_i) \right) + b \norm{d(S,U) - d(S)}_2 + 2\abs{B(W,\mu)} + O_p(1/n).
\end{equation*}
\end{lemma}

Then given \cref{lem:non-iid-cmse-bound}, it follows that \Cref{lem:non-iid-bias-bound} gives a bound on the irreducible squared bias of $\hat\tau_W$ as $n \to \infty$.

We note that this bound is an explicit function of the difference between $P(U_i \mid Z_i)$ and $P(U_i \mid Z_{1:n})$ for each $i$, and the difference between $d(S)$ and $d(S,U)$. Furthermore in the iid confounder case this bound on the squared bias vanishes to zero as $n \to \infty$, as long as $J_{4\sigma^2}(W,\mu) = o_p(1)$, as is ensured by our balancing theory under the assumptions in \cref{sec:alg}. This provides some concrete justification for our intuition that in ``near-iid'' settings our estimator should be close to consistent.

Next, we observe that if one uses the supremum norm for $D_{\mathcal F}$ then the corresponding IPM is total variation distance, and we easily satisfy the theorem requirements with $F=b$ given \cref{asm:bounded-mean}. However the general form of the theorem allows for alternate tighter bounds in terms of weaker IPMs under assumptions on the norm of $\mu$. In particular if we assume $\mu$ is contained in an RKHS, as in the case of our kernel-based algorithm, another natural choice for $D_{\mathcal F}$ would be the corresponding maximum mean discrepancy (MMD).

In addition we note that given some assumptions on $\mu$, all terms in the bound can be estimated in practice for a given weighted estimator. This means practitioners can estimate the bound under different non-iid model assumptions and assumptions on $\mu$, in order to perform sensitivity analysis. Furthermore given the dependence of the first term in this estimator on $\norm{W}_{\infty}$, this may motivate additional regularization on $W$ in non-iid settings. However we leave further exploration of this idea to future work.

Finally, we provide the cautionary note that in the non-iid setting the identification assumptions for $d(S)$ are invalid, and therefore our proposed algorithm for learning the state density ratio may be inconsistent. Therefore the $d(S_i)$ terms in the above theorem should be interpreted as coming from the possibly biased $d$ function used by the optimal balancing algorithm. The theorem then provides an explicit bound on the incurred bias due to this. Note however that in the case that $d(S) \approx d(S,U)$, the estimating equations in \cref{sec:state-density-ratio} are approximately correct, so we do not expect this to be a major issue in practice. This is further justified by the strong positive results of our sensitivity experiments.

\subsection{Omitted Proofs for Sensitivity Theory}
\label{apx:sensitivity-proofs}

\begin{proof}[Proof of \cref{lem:non-iid-cmse-bound}]

First we define sample average policy effect slightly differently for the non-iid setting, as:
\begin{equation*}
  \sape^*(\pi_e) = \frac{1}{n} \sum_{i=1}^n \sum_{a=1}^m d(S_i,U_i) \pi_e(a \mid S_i, U_i) \mu_a(S_i, U_i).
\end{equation*}

Again, following the derivation in \cref{sec:alg} we have $\e[\sape(\pi_e)] = \pv(\pi_e)$. Given this and \cref{asm:ergodic,asm:mixing,asm:bounded-mean,asm:overlap-non-iid}, it is clear that the conditions of \cref{lem:mixing} apply to $\e_b[(\sape^*(\pi_e) - \pv(\pi_e))^2]$, so this term must be $O(1/n)$. Thus by Markov's inequality and the law of total expectation we have $\e[(\sape^*(\pi_e) - \pv(\pi_e))^2 \mid Z_{1:n}] = O_p(1/n)$. Then, using the fact that $(x+y)^2 \leq 2x^2 + 2y^2$, we have
\begin{equation*}
  \e[(\hat\tau_W - \pv(\pi_e))^2 \mid Z_{1:n}] \leq 2\e[(\hat\tau_W - \sape^*(\pi_e))^2 \mid Z_{1:n}] + O_p(1/n).
\end{equation*}

Next, we perform a bias variance decomposition of the RHS of this bound as follows: 
\begin{align*}
  \e[(\hat\tau_W - \sape^*(\pi_e))^2 \mid Z_{1:n}] &= \e[ \e[(\hat\tau_W - \sape^*(\pi_e))^2 \mid Z_{1:n}, U_{1:n}] \mid Z_{1:n}] \\
  &= \e[ \e[\hat\tau_W - \sape^*(\pi_e) \mid Z_{1:n}, U_{1:n}]^2 \mid Z_{1:n}] \\
  &\qquad + \e[ \var[\hat\tau_W - \sape^*(\pi_e) \mid Z_{1:n}, U_{1:n}] \mid Z_{1:n} ] \\
  &= \xi_1^* + \xi_2^*,
\end{align*}
and we additionally define
\begin{align*}
    \zeta_{ia}^* &= W_i \delta_{A_i a} R_i - d(S_i, U_i) \pi_e(a \mid S_i, U_i) \mu_a(S_i, U_i) \\
    \zeta_i^* &= \sum_{a=1}^m \zeta_{ia}^* = W_i R_i - d(S_i, U_i) \sum_{a=1}^m \pi_e(a \mid S_i, U_i) \mu_a(S_i, U_i).
\end{align*}

Given this, the first term of the above bias variance decomposition can be broken down as:
\begin{align*}
  \xi_1^* &= \e \left[ \left( \frac{1}{n} \sum_{i=1}^n \sum_{a=1}^m \zeta_{ia}^* \right)^2 \mathrel{\Big|} Z_{1:n} \right] \\
  &= \e \left[ \frac{1}{n} \sum_{i=1}^n \sum_{a=1}^m \zeta_{ia}^* \mathrel{\Big|} Z_{1:n} \right]^2 + \var \left[ \frac{1}{n} \sum_{i=1}^n \sum_{a=1}^m \zeta_{ia}^* \mathrel{\Big|} Z_{1:n} \right] \\
  &= \left( \frac{1}{n} \sum_{i=1}^n \sum_{a=1}^m \e[ f_{ia}^* \mu_a(S_i,U_i) \mid Z_{1:n}] \right)^2 + \frac{1}{n^2} \var[ \sum_{i=1}^n \zeta_i^* \mid Z_{1:n} ] \\
  &\leq \left( \frac{1}{n} \sum_{i=1}^n \sum_{a=1}^m \e[ f_{ia}^* \mu_a(S_i,U_i) \mid Z_{1:n}] \right)^2 + \frac{2\sigma^2}{n^2} \sum_{i=1}^n W_i^2 \\
  &\qquad + 2 \var \left[ \frac{1}{n} \sum_{i=1}^n d(S_i,U_i) \sum_{a=1}^m \pi_e(a \mid S_i, U_i) \mu_a(S_i, U_i) \mathrel{\Big|} Z_{1:n} \right] \\
  &= B^*(W,\mu)^2 + \frac{2\sigma^2}{n^2} \norm{W}^2  + 2 \var \left[ \frac{1}{n} \sum_{i=1}^n d(S_i,U_i) \sum_{a=1}^m \pi_e(a \mid S_i, U_i) \mu_a(S_i, U_i) \mathrel{\Big|} Z_{1:n} \right].
\end{align*}
Similarly, we bound the the second error term $\xi_2^2$ as:
\begin{align*}
  \xi_2^* &= \e \left[ \var \left[ \frac{1}{n} \sum_{i=1}^n \zeta_i^* \mathrel{\Big|} Z_{1:n}, U_{1:n} \right] \mathrel{\Big|} Z_{1:n} \right] \\
  &\leq \e \left[ \frac{2\sigma^2}{n^2} \sum_{i=1}^n W_i^2 + 2 \var \left[ \frac{1}{n} \sum_{i=1}^n  d(S_i,U_i) \sum_{a=1}^m \pi_e(a \mid S_i,U_i) \mu_a(S_i,U_i) \mathrel{\Big|} Z_{1:n}, U_{1:n} \right] \mathrel{\Big|} Z_{1:n} \right] \\
  &\leq \frac{2\sigma^2}{n^2} \norm{W}^2 + 2 \var \left[ \frac{1}{n} \sum_{i=1}^n d(S_i,U_i) \sum_{a=1}^m \pi_e(a \mid S_i, U_i) \mu_a(S_i, U_i) \mathrel{\Big|} Z_{1:n} \right],
\end{align*}
where the final inequality follows from the law of total variance. Next, applying \cref{asm:ergodic,asm:mixing,asm:bounded-mean,asm:overlap-non-iid}, it clearly follows from \cref{lem:mixing} that
\begin{equation*}
    \var \left[ \frac{1}{n} \sum_{i=1}^n d(S_i,U_i) \sum_{a=1}^m \pi_e(a \mid S_i, U_i) \mu_a(S_i, U_i) \right] = O(1/n),
\end{equation*}
and therefore by Markov's inequality the corresponding conditional variance is $O_p(1/n)$.

Putting the above bounds together we get
\begin{equation*}
  \e[(\hat\tau_W - \pv(\pi_e))^2 \mid Z_{1:n})] \leq 2 \left( B^*(W,\mu)^2 + \frac{4 \sigma^2}{n^2} \norm{W}^2 \right) + O_p(1/n),
\end{equation*}

which gives us our required result immediately.

\end{proof}

\begin{proof}[Proof of \cref{lem:non-iid-bias-bound}]

First we can obtain the bound
\begin{align*}
    \epsilon(W,\mu)^2 &= \abs{B^*(W,\mu)^2 - B(W,\mu)^2} \\
    &= \abs{B^*(W,\mu) - B(W,\mu)} \abs{B^*(W,\mu) + B(W,\mu)} \\
    &= \abs{B^*(W,\mu) - B(W,\mu)} \abs{2B(W,\mu) + (B^*(W,\mu) - B(W,\mu))} \\
    &\leq \abs{B^*(W,\mu) - B(W,\mu)} (2\abs{B(W,\mu)} + \abs{B^*(W,\mu) - B(W,\mu)}) \\
    &\leq (2\abs{B(W,\mu)} + \abs{B^*(W,\mu) - B(W,\mu)})^2.
\end{align*}

Next, let $b$ be a constant such that $\abs{\mu_a(s,u)} \leq b \ \forall a,s,u$, which by \cref{asm:bounded-mean} must exist, and define the notation shorthand
\begin{align*}
    e_i(\cdot) &= \e[\cdot \mid Z_i] \\
    e_{1:n}(\cdot) &= \e[\cdot \mid Z_{1:n}].
\end{align*}
Given this we can obtain the bound
\begin{align*}
    \abs{B^*(W,\mu) - B(W,\mu)} &\leq \left| \frac{1}{n} \sum_{i=1}^n (e_{1:n} - e_i) \left( \sum_{a=1}^m f^*_{ia} \mu_a(S_i, U_i) \right) \right| \\
    &\qquad + \left| \frac{1}{n} \sum_{i=1}^n e_i \left( \sum_{a=1}^m (f^*_{ia} - f_{ia}) \mu_a(S_i, U_i) \right) \right| \\
    &\leq \frac{1}{n} \sum_{i=1}^n \abs{W_i} \left| (e_{1:n} - e_i) \sum_{a=1}^m \delta_{A_i a} \mu_a(S_i, U_i) \right| \\
    &\qquad + \frac{1}{n} \sum_{i=1}^n \left| (e_{1:n} - e_i) \sum_{a=1}^m \pi_e(a \mid S_i, U_i) \mu_a(S_i, U_i) \right| \\
    &\qquad + \left| \e \left[ (d(S_i,U_i) - d(S_i)) \sum_{a=1}^m \pi_e(a \mid S_i, U_i) \mu_a(S_i, U_i) \right] \right| + O_p(1/n) \\
    &\leq \frac{F}{n} \sum_{i=1}^n (\abs{W_i} + 1) D_{\mathcal F}(\varphi_i,\varphi_i^*) + b \e[(d(S,U) - d(S))^2]^{1/2} + O_p(1/n) \\
    &\leq \frac{F}{n} \sum_{i=1}^n (\abs{W_i} + 1) D_{\mathcal F}(\varphi_i,\varphi_i^*) + b \norm{d(S,U) - d(S)}_2 + O_p(1/n),
\end{align*}
where in the second inequality we apply the Markov chain law of large numbers, in the third and final inequalities we apply Cauchy Schwartz and our $\norm{\cdot}_{\mathcal F}$ bound assumptions. Putting the above together, we obtain the final bound:
\begin{equation*}
    \epsilon(W,\mu) \leq F \left( \frac{1}{n} \sum_{i=1}^n (\abs{W_i}+1) D_{\mathcal F}(\varphi_i,\varphi_i^*) \right)  + b(\norm{d(S,U) - d(S)}_2 + 2 \norm{B(W,\mu)} + O_p(1/n).
\end{equation*}

\end{proof}

\section{Discussion of Nuisance Estimation}

We discuss here some of the existing theory regarding the estimation of the posterior distributions $\varphi$, and the state density ratio $d$, including the assumptions neeed for identification and for the rates of convergence required by our theory.

\subsection{Estimation of Confounder Posterior Distribution}
\label{apx:phi-estimation}

We provide some discussion here for convergence rates of $D_{\mathcal F}(\hat\varphi(Z), \varphi(Z))$ in the case where $\norm{f}_{\mathcal F} = \norm{f}_{\infty}$, which corresponds to total variation distance, since this metric dominates most other integral probability metrics (IPMs) of interest.

First, for any given $z$ we can obtain the bound
\begin{align*}
    D_{\mathcal F}(\varphi(z), \hat\varphi(z)) &= \sup_{\norm{f}_\infty} \left| \int f d \varphi(z) - \int f d \hat\varphi(z) \right| \\
    &\leq \left| \int \varphi(z)(u) - \hat\varphi(z)(u) du \right| \\
    &\leq \sup_{u \in \mathcal U} \abs{\varphi(z)(u) - \hat\varphi(z)(u)} \int du.
\end{align*}

Now, under the assumption that $\mathcal U$ is compact, we have $\int du < \infty$, so it is sufficient to consider the convergence rate of $\sup_{u \in \mathcal U} \abs{\varphi(z)(u) - \hat\varphi(z)(u)}$. We analze this convergence for multiple cases below.

\subsubsection{Discrete States and Confounders}

The simplest case to consider here is the case where both $S$ and $U$ are discrete, as in our experiments. Under this assumption, the above bound translates to requiring that $\abs{\varphi(z)(u) - \hat\varphi(z)(u)}$ converges sufficiently fast for each $U$ and $Z$ level. Fortunately, in this case the probabilities $P(U \mid Z)$ are given by parameters in some parametric latent variable model, which can be fit using approaches such as expectation maximization (EM) \citep{dempster1977maximum}, Bayesian estimators \citep{lehmann2006theory}, or spectral methods \citep{hsu2009spectral,shaban2015learning}. In particular, maximum likelihood-based approaches such as the EM algorithm, are known to be efficient and achieve the $O_p(n^{-1/2})$-convergence required for $O_p(n^{-1/2})$ OPE consistency \citep{van2000asymptotic}. Note that in the case of EM this depends on solving the difficult non-convex optimization problem, however this challenge may be mitigated by initializing EM with some non-local optimization method \citep{shaban2015learning}. This analysis depends on the assumption that the confounder model is well-specified (i.e. confounders are actually discrete, and we do not underestimate the number of confounder levels). In addition it depends on standard identifiability conditions needed for latent variable models in general \citep{dempster1977maximum}.

\subsubsection{Continuous States and Discrete Confounders}

In this next case $U$ is still assumed to be discrete, so again it is sufficient to ensure that for any given $z$, we have that $\abs{\varphi(z)(u) - \hat\varphi(z)(u)}$ converges sufficiently fast for each $u \in \mathcal U$. If we assume a parametric model such that $\varphi(z) = \varphi_{\theta_0}(z)$ for some finite-dimensional parameter space $\Theta$ and some $\theta_0 \in \Theta$, then $\theta_0$ can be estimated using the kinds of approaches described in the previous section. Under standard correct-specification and identifiability assumptions it easily follows that we can obtain $O_p(n^{-1/2})$ consistency for estimating $\theta_0$. Then under some smoothness assumptions of $\varphi_{\theta}(z)$ (e.g. locally Lipschitz at $\theta_0$), it follows that $\abs{\varphi(z)(u) - \hat\varphi(z)(u)} = O_p(n^{-1/2})$, and therefore we can obtain the same parametric rate for our policy value estimate. Alternatively, if we assume some kind of semi- or non-parametric model for $\varphi(z)$, then we may still be able to estimate $\varphi(z)(u)$ at some rate in between $O_p(n^{-1/4})$ and $O_p(n^{-1/2})$ using machine learning methods, under some smoothness assumptions, as is standard for flexible nuisance estimation in causal inference (see for example discussion in \citet{chernozhukov2016double}).

\subsubsection{Continuous States and Confounders}

In this final most general case, we can again consider estimating $\varphi(z)$ either by assuming a parametric model, or using flexible machine learning methods that exploit smoothness. Again this can result in estimates of $\varphi(z)(u)$ that are either $O_p(n^{-1/2})$-consistent under parametric assumptions, or consistent at some slower rate under more general smoothness assumptions. This allows us to guarantee convergence for any fixed $u \in \mathcal U$, however in this case we have the additional complexity that the space $\mathcal U$ is not finite, and therefore we need to establish the convergence of $\sup_{u \in \mathcal U} \abs{\varphi(z)(u) - \hat\varphi(z)(u)}$. Let $Q_n(u) = (\varphi(z) - \hat\varphi(z)) / r_n$. Then if we assume that $Q_n$ is uniformly sub-Gaussian in $\mathcal U$ for every $n \in \mathbb{N}$ (that is there exists some semi-metric $d$ on $\mathcal U$ such that $P(\abs{Q_n(u) - Q_n(u')} > x) \leq 2\exp(-\frac{1}{2} x^2 / d(u,u')^2)$ for every $n \in \mathbb{N}$, $u,u' \in \mathcal U$), it follows easily from standard chaining arguments \citep[Corollary 8.5 and Theorem 2.1]{kosorok2007introduction} that $\sup_{u \in \mathcal U} \abs{\varphi(z)(u) - \hat\varphi(z)(u)} = O_p(r_n)$. Note that following standard empirical process theory arguments, this required sub-Gaussian assumption may be justified based on compactness of $\mathcal U$ and Lipschitz continuity assumptions.

\subsection{Estimation of State Density Ratio}
\label{apx:d-estimation}

Here we discuss the rate of convergence of the state density ratio $d$. First, in the case that $\mathcal S$ is discrete, as in our experiments, the variational GMM algorithm we proposed reduces to a standard efficient GMM algorithm for a finite number of parameters (in the case that $\mathcal S = [n_s]$, these parameters are $d(1), \ldots, d(n_s)$) as discussed in \cref{apx:gmm-derivation}. These algorithms are known to be semi-parametrically efficient, with $O_p(n^{-1/2})$ consistency \citep{hansen1982large}, as required for $O_p(n^{-1/2})$-consistent estimation of $\pv(\pi_e)$.

In the more general case, where $\mathcal S$ is continuous, the theory on the rate of convergence of $\hat d$ is less clear. If we replaced the RKHS class for $\mathcal D$ used in our algorithm with a parametric class, then under an identifiability assumption on the class $\mathcal H$ (that it is sufficiently rich to identify $d$), and the assumption that $\mathcal H$ has a finite basis (such as in the case of a polynomial kernel), then again this corresponds to a standard efficient GMM estimate and $O_p(n^{-1/2})$-consistency would follow from standard GMM theory \citep{hansen1982large}. On the other hand in the more general case we consider in \cref{sec:state-density-ratio}, where $\mathcal D$ and $\mathcal H$ are both flexible potentially non-parametric function classes, consistency of $\hat d$ could be established using a proof almost identical to that in \citet{bennett2019deep}. However the rate of convergence in general settings where $\mathcal D$ and $\mathcal H$ can both be arbitrary RKHSs is unclear, and we leave this problem to future work.

\section{Estimation using Universally-Approximating Function Class}
\label{apx:universally-approximating}

Suppose that we have some series of function classes $\mathcal G_i$ for $i \in \{1,2,\ldots\}$, such that for any vector-valued function $g$ and $a \in [m]$ we have
\begin{equation*}
    \lim_{i \to \infty} \inf_{g' \in \mathcal G_i} \norm{g_a - g'_a}_{\infty} = 0.
\end{equation*}
We call such a function class universally approximating, and note that the RKHS described in \cref{lem:kernel-g} that we use in our methodology satisfies this definition for many commonly used classes of kernels, such as the Gaussian kernel with shrinking variance parameter \citep{mendelson2003performance}.

Then from \cref{thm:bound-convergence}, it easily follows that, by choosing using sufficiently large $i$ we can ensure $\sup_{g \in \mathcal G_i} J_n(W^*, g) = O_p(n^{-1/2}) + \epsilon$ for any given $\epsilon > 0$, where the constants in the $O_p(n^{-1/2})$ term possibly depend on $\epsilon$. Given this by choosing increasingly larger $i$ as $n \to \infty$, we can ensure that $\sup_{g \in \mathcal G_i} J_n(W^*, g) = O_p(r_n)$ for some sequence $r_n \to 0$, where the unkonwn rate $r_n$ depends on the relation between the constant in the $O_p(n^{-1/2})$ term and $\epsilon$ in the previous bound, and also on the rate at which we increase $i$ as $n \to \infty$. Finally then appealing to \cref{thm:mse-bound}, it follows that we can achieve $O_p(r_n^{1/2})$-consistency, for the unknown rate $r_n$.

\section{Derivation of Algorithm for State Density Ratio Estimation}
\label{apx:gmm-derivation}

We discuss here the theoretical derivation of the variational GMM algorithm presented in \cref{sec:state-density-ratio} for state density ratio estimation.

First, we observe that it follows easily from a generalization of \citet[Lemma 1]{bennett2019deep} (replacing the instrumental variable regression conditional moment restrictions there with the state density ratio conditional moment restrictions) that if $\mathcal H$ is the vector space spanned by functions $\{h_1, \ldots, h_k\}$, and $\mathcal D$ is given by some parametric class, then the estimator
\begin{equation*}
    \hat d = \argmin_{d \in \mathcal D} \sup_{h \in \mathcal H, c \in \mathbb R} U_n(d, \tilde d, h, c)
\end{equation*}
is exactly the same as the standard optimally-weighted GMM estimator \citep{hansen1982large} given by the $k+1$ standard moment restrictions
\begin{align*}
    \e[h_i(S')(d(S) \beta(Z) - d(S'))] &= 0 \ \forall i \in [k] \\
    \e[d(S) - 1] &= 0.
\end{align*}

Given standard regularity assumptions, that $d \in \mathcal D$, the $k+1$ moment restrictions are sufficient to uniquely identify $d$, and that the parametric class for $\mathcal D$ is sufficiently smooth, then it follows from standard theory that this estimator is root-$n$ consistent and asymptotically normal, and if the prior estimate $\tilde d$ is consistent then the estimator is statistically efficient relative to all other estimators based on these $k+1$ moment conditions. Note that given the above, efficiency is easily ensured by running the adversarial optimization at least twice, starting with an initial arbitrary guess for $\tilde d$ and then each time using the previous iterate estimate $\hat d$ for $\tilde d$, as proposed in \cref{sec:state-density-ratio}.

Given this, it is natural to consider extending this standard GMM estimator by replacing $\mathcal D$ and $\mathcal H$ by sufficiently regularized flexible function classes, such as neural networks or RKHSs. This is motivated by wanting to avoid the known curse of dimensionality issues of seive estimators using increasingly large numbers of standard moment conditions. Previously \citet{bennett2019deep} proposed to use such an estimator for the instrumental variable regression problem using neural networks for both function classes. On the other hand we propose to use RKHSs, which has the nice benefit that the optimization can be performed analytically by appealing to the representor theorem (as discussed in \cref{apx:representer-details}).

\section{Additional Methodology Details}

\subsection{Details on Calculating State Density Ratio}
\label{apx:representer-details}

We provide details here for the state density ratio calculations, in the case that $\Sset$ and $\Uset$ are discrete, as in our experiments. Specifically, we assume that $\Sset = \{1,\ldots, N_S\}$ and $\Uset = \{1,\ldots,N_U\}$ for some integers $N_S$ and $N_U$. Applying the representer theorem, we can represent an optimal solution to both $h$ and $d$ in terms of $N_S$ parameters. In addition we let $\beta_i = \sum_{u=1}^{N_U} \hat\varphi(u \mid Z_i) \pi_e(A_i \mid S_i, u) / \pi_b(A_i \mid S_i, u)$, where $\hat\varphi$ is our oracle for calculating posterior probabilities. Recall that the objective is:
\begin{align*}
    U(d,\tilde d, f,c,c') &= \frac{1}{n} \sum_{i=1}^n ((\beta_i d(S_i) - d(S_i'))h(S_i') + c(d(S_i) - 1) + c'(d(S_i) - 1)) \\
    &\qquad - \frac{1}{4n} \sum_{i=1}^n ((\beta_i \tilde d(S_i) - \tilde d(S_i'))h(S_i') + c(\tilde d(S_i) - 1) + c'(\tilde d(S_i) - 1))^2
\end{align*}
where $h(s) = \sum_{x=1}^k \alpha_x K_h(x,s)$, and $d(s) = \sum_{x=1}^k \gamma_x K_d(x,s)$. Note that unlike in the prose of our paper we separately enforce the moment conditions $\e[d(S)] = 0$ and $\e[d(S')] = 0$. Although this is theoretically unnecessarily, and in many cases redundant since the set of observed $S$ values is almost identical to the set of observed $S'$ values, we do this for generality in the case that we sampled data with some thinning.

First consider the $\sup_{h,c,c'}$ sub-problem. From the above representations we can obtain.
\begin{align*}
    U(d,\tilde d,h,c,c') &= \frac{1}{n} \sum_{i=1}^n \sum_{x=1}^{N_S} (\beta_i d(S_i) - d(S_i')) \alpha_x K_h(x,S_i') \\
    &\qquad + c \frac{1}{n} \sum_{i=1}^n(d(S_i) - 1) + c' \frac{1}{n} \sum_{i=1}^n(d(S_i') - 1) \\
    &\qquad - \frac{1}{4n} \sum_{i=1}^n \left(\sum_{x=1}^{N_S} \left((\beta_i \tilde d(S_i) - \tilde d(S_i')) \alpha_x K_h(x,S_i')\right) + c(\tilde d(S_i) - 1) + c'(\tilde d(S_i') - 1)\right)^2 \\
    &= \frac{1}{n} \sum_{i=1}^n \sum_{x=1}^{N_S} (\beta_i d(S_i) - d(S_i')) \alpha_x K_h(x,S_i') \\
    &\qquad + c \frac{1}{n} \sum_{i=1}^n(d(S_i) - 1) + c' \frac{1}{n} \sum_{i=1}^n(d(S_i') - 1) \\
    &\qquad - \frac{1}{4n} \sum_{i=1}^n \sum_{x,y=1}^{N_S} (\beta_i \tilde d(S_i) - \tilde d(S_i'))^2 \alpha_x \alpha_y K_h(x,S_i') K_h(y,S_i') \\
    &\qquad - \frac{1}{4n} \sum_{i=1}^n \sum_{x=1}^{N_S} 2 (\beta_i \tilde d(S_i) - \tilde d(S_i')) \alpha_x K_h(x,S_i') (c(\tilde d(S_i) - 1) + c'(\tilde d(S_i') - 1)) \\
    &\qquad - \frac{1}{4n} \sum_{i=1}^n (c(\tilde d(S_i) - 1) + c'(\tilde d(S_i') - 1))^2.
\end{align*}
Next define:
\begin{align*}
    g_x &= \sum_{i : S_i' = x} (\beta_i \tilde d(S_i) - \tilde d(x))^2 \\
    \phi_x^{(1)} &= \sum_{i : S_i' = x} (\beta_i \tilde d(S_i) - \tilde d(x))(\tilde d(S_i) - 1) \\
    \phi_x^{(2)} &= \sum_{i : S_i' = x} (\beta_i \tilde d(S_i) - \tilde d(x))(\tilde d(x) - 1).
\end{align*}
Then it follows easily from the above that we have:
\begin{equation*}
    U(d,\tilde d,h,c,c') = (\alpha,c,c')^T q - \frac{1}{4} (\alpha,c,c')^T Q (\alpha,c,c'),
\end{equation*}
where the vector $q$ and symmetric matrix $Q$ are given by:
\begin{align*}
    q_x &= \frac{1}{n} \sum_{i=1}^n (\beta_i d(S_i) - d(S_i')) K_h(x, S_i') \qquad \forall x \in [m] \\
    q_{m+1} &= \frac{1}{n} \sum_{i=1}^n (d(S_i) - 1) \\
    q_{m+2} &= \frac{1}{n} \sum_{i=1}^n (d(S_i') - 1) \\
    Q_{x,y} &= \frac{1}{n} \sum_z g_z K_h(x,z) K_h(y,z) \qquad \forall x,y \in [m] \\
    Q_{x,m+1} &= \frac{1}{n} \sum_z \phi_z^{(1)} K_h(x,z) \qquad \forall x \in [m] \\
    Q_{x,m+2} &= \frac{1}{n} \sum_z \phi_z^{(2)} K_h(x,z) \qquad \forall x \in [m] \\
    Q_{m+1,m+1} &= \frac{1}{n} \sum_{i=1}^n (\tilde d(S_i) - 1)^2 \\
    Q_{m+1,m+2} &= \frac{1}{n} \sum_{i=1}^n (\tilde d(S_i) - 1)(\tilde d(S_i') - 1) \\
    Q_{m+1,m+1} &= \frac{1}{n} \sum_{i=1}^n (\tilde d(S_i') - 1)^2.
\end{align*}

Assuming $Q$ is positive definite, it follows easily by taking derivatives that the objective is maximized by $\alpha,c,c' = 2 Q^{-1} q$, which gives:
\begin{equation*}
    \sup_{h \in \Fset_h, c \in \mathbb R, c' \in \mathbb R} O(d,h,c,c') = q^T Q^{-1} q.
\end{equation*}

In case that $Q$ is not PD and/or we wish to regularize, we replace $Q$ with $Q + D$ for some PD matrix $D$. In particular we use $D = \text{BlockDiagonal}(\lambda_h K^{(h)}, \lambda_c, \lambda_c)$, where $K^{(h)}_{xy} = K_h(x,y)$, and we note that using this $D$ is equivalent by Lagrange duality to restricting the RKHS norm of $h$ and the euclidean norm of $c$ and $c$'.

Now we consider the outside minimization problem. First we can note that $Q$ (or $Q + D$) does not depend on $d$, so it can be treated as a constant for this outside problem.

Let $n(a,x,y)$ be the number of data points where $A=a,S_i=x,S_i'=y$, $n(x)$ be the number of data points where $S_i=x$, and $n'(x)$ be the number of data points where $S_i'=x$. Plugging the above solution $h,c,c'$ into our equation for $q$, for $x \in [N_S]$ we get:
\begin{align*}
    q_x &= \frac{1}{n} \sum_{a,y,z} n(a,y,z) (\beta(a,y,z) d(y) - d(z)) K_h(x,z) \\
    &= \frac{1}{n} \sum_{y} \left( \sum_{a,z} n(a,y,z) \beta(a,y,z) K_h(x,z) \right) d(y) - \frac{1}{n} \sum_z n'(z) K_h(x,z) d(z) \\
    &= \frac{1}{n} \sum_y \psi(x,y) d(y),
\end{align*}
where
\begin{equation*}
    \psi(x,y) = \sum_{a,z} \left( n(a,y,z) \beta(a,y,z) K_h(x,z) \right) - n'(y) K_h(x,y).
\end{equation*}
In addition we can easily obtain:
\begin{align*}
    q_{m+1} &= \frac{1}{n} \sum_x n(x)(d(x) - 1) \\
    q_{m+2} &= \frac{1}{n} \sum_x n'(x)(d(x) - 1).
\end{align*}

Given the above we can derive:
\begin{align*}
    \sup_{h,c,c'} O(d,\tilde d,h,c,c') &= q^T (Q+D)^{-1} q \\
    &= \frac{1}{n^2} \sum_{x,y} (Q+D)^{-1}_{x,y} \sum_{z,w} \psi(x,z)d(z) \psi(y,w)d(w) \\
    &\qquad + \frac{2}{n^2} \sum_x (Q+D)^{-1}_{x,m+1} (\sum_y \psi(x,y)d(y))(\sum_y n(y)(d(y) - 1)) \\
    &\qquad + \frac{2}{n^2} \sum_x (Q+D)^{-1}_{x,m+2} (\sum_y \psi(x,y)d(y))(\sum_y n'(y)(d(y) - 1)) \\
    &\qquad + \frac{1}{n^2} (Q+D)^{-1}_{m+1,m+1} (\sum_y n(y)(d(y) - 1))^2 \\
    &\qquad + \frac{2}{n^2} (Q+D)^{-1}_{m+1,m+2} (\sum_y n(y)(d(y) - 1))(\sum_y n'(y)(d(y) - 1)) \\
    &\qquad + \frac{1}{n^2} (Q+D)^{-1}_{m+2,m+2} (\sum_y n'(y)(d(y) - 1))^2.
\end{align*}
This can be re-written as $\sum_{x,y} B_{x,y} d(x)d(y) + \sum_x b_x d(x) + C$, where $C$ is constant in $d(1),\ldots,d(N_S)$, where we define the symmetric matrix $B$ and the vector $b$ by:
\begin{align*}
    B_{x,y} &= \frac{1}{n^2} \sum_{z,w} (Q+D)^{-1}_{z,w} \psi(z,x) \psi(w,y) \\
    &\qquad + \frac{1}{n^2} \sum_z (Q+D)^{-1}_{z,m+1} (\psi(z,x)n(y) + \psi(z,y)n(x)) \\
    &\qquad + \frac{1}{n^2} \sum_z (Q+D)^{-1}_{z,m+2} (\psi(z,x)n'(y) + \psi(z,y)n'(x)) \\
    &\qquad + \frac{1}{n^2} (Q+D)^{m+1,m+1} n(x)n(y) \\
    &\qquad + \frac{1}{n^2} (Q+D)^{m+1,m+2} n(x)n'(y) + n'(x)n(y) \\
    &\qquad + \frac{1}{n^2} (Q+D)^{m+2,m+2} n'(x)n'(y) \\
    b_x &= -\frac{2}{n} \sum_y ((Q+D)^{-1}_{y,m+1} + (Q+D)^{-1}_{y,m+2}) \psi(y,x) \\
    &\qquad - \frac{2}{n} (Q+D)^{-1}_{m+1,m+1} n(x) \\
    &\qquad - \frac{2}{n} (Q+D)^{-1}_{m+1,m+2} (n(x) + n'(x)) \\
    &\qquad - \frac{2}{n} (Q+D)^{-1}_{m+2,m+2} n'(x).
\end{align*}

Given this and our representation of $d$, we have
\begin{equation*}
    \sup_{f,c,c'} O(d,f,c,c') = \gamma^T K^{(d)} B K^{(d)} \gamma + b^T K^{(d)} \gamma + C,
\end{equation*}
where $K^{(d)}_{xy} = K_d(x,y)$. Thus assuming that $K^{(d)} B K^{(d)}$ is PD we easily have that the optimal value optimizing over $\Fset_d$ is given by $\gamma = -\frac{1}{2} (K^{(d)} B K^{(d)})^{-1} K^{(d)} b$. Again, if the matrix is not PD or we wish to regularize, we replace it in our with $K^{(d)} B K^{(d)} + D_d$, where $D_d = \lambda_d K^{(d)}$.

Finally, given $\gamma = -\frac{1}{2} (K^{(d)} B K^{(d)} + \lambda_d K^{(d)})^{-1} K^{(d)} b$, our output state density ratio function is given by
\begin{equation*}
    \hat d(s) = \sum_{x=1}^{N_S} \gamma_x K_d(x,s).
\end{equation*}

\subsection{Details on Calculating Optimal Weights}
\label{apx:weights-solving-details}

In all of our experiments $\Sset$ and $\Uset$ are discrete. As above we denote $\Sset = \{1,\ldots,N_S\}$, and $\Uset = \{1,\ldots,N_U\}$. Given this, we compute all $\e[\phi(U_i, \tilde U_j) \mid Z_i, Z_j]$-style terms appearing in \cref{thm:w-objective} according to
\begin{equation*}
    \e[\phi(U_i, \tilde U_j) \mid Z_i, Z_j] = \sum_{u,u'=1}^{N_U} \hat\varphi(u;Z_i) \hat\varphi(u';Z_j) \phi(u,u'),
\end{equation*}
where $\hat\varphi$ is our approximate oracle for the posterior distribution of our confounders.

Now, it is trivial to verify that an optimal solution to our quadratic objective will always be given by choosing $W_{1:n}$ such that $W_i = W_j$ whenever $Z_i = Z_j$. Therefore given that $\Sset$ and $\Uset$ are discrete we only need to calculate $C$ separate weights in our optimization problem, where $C$ is the number of distinct $Z$ values observed. Specifically, we can set up an equivalent optimization as follows: first let $\{Z'_1,\ldots,Z'_C\}$ be the set of unique observed $Z$ values, and $\{N_1,\ldots,N_C\}$ be the number of times each was observed. Then we calculate $C \times C$ matrices $k$ and $G$ and vector length-$C$ vector $g$ according to
\begin{align*}
    k'_{ij} &= K((S'_i,U'_i),(S'_j,\tilde U'_j)) \\
    G'_{ij} &= N_i N_j \delta_{A'_i A'_j} \e[k_{ij} \mid Z'_i, Z'_j] + N_i \lambda \delta_{i j} \\
    g'_i &= N_i d(S'_j) \e[\pi_e(A'_i \mid S'_j, U'_j) k_{ij} \mid Z_i Z_j]
\end{align*}
and calculate $W' = G'^{-1} g'$. Then we finally compute the final length-$n$ vector of weights by indexing this length-$C$ vector. Specifically define $\nu : [n] \mapsto [C]$ such that $Z_i = Z'_{\nu(i)} \ \forall i \in [n]$. Then the final weights we return are given by $W_i = W'_{\nu(i)}$.

\section{Additional Experiment Details}

\subsection{Baseline Descriptions}
\label{apx:baselines}

\paragraph{Direct Method:} This method works by using the approximate confounder model to directly fit an outcome model. Specifically, first we use the confounder-imputed dataset to fit a model $\hat\mu$ for $\mu$ via regressing $R$ on $(S,\hat U)$ for each $a \in [m]$. Given that our experiments work with discrete states and confounders, this is done simply by averaging the observed reward for each possible $(s,u)$ pair. Then we use the estimated outcome model, stationary density ratio, and confounder model to directly estimate $\pv(\pi_e)$, according to 
    \begin{align}
     \hat{\tau}_{\rm DM}^{(i)} &= \sum_{u,a} \hat\varphi(u \mid Z_i) \pi_e(a \mid S_i, u) \hat \mu_a(S_i, u)\nonumber \\
     \hat{\tau}_{\rm DM} &= \frac{1}{n} \sum_{i=1}^n \hat{d}(S_i) \hat{\tau}_{\rm DM}^{(i)}.
    \end{align}

\paragraph{Doubly Robust} This method combines the Direct Method and our weighted estimator approach. Specifically given weights $W_{1:n}$ and an outcome model $\hat \mu$ fit as above, we calculate 
    \begin{align}
        \hat\tau_{\rm DR}^{(i)} &= \sum_u \hat \varphi(u \mid Z_i) \hat \mu_{A_i}(S_i, u) \nonumber \\
        \hat\tau_{\rm DR} &= \hat \tau_{\rm DM} + \frac{1}{n} \sum_{i=1}^n W_i (R_i - \hat{\tau}_{\rm DR}^{(i)}).
    \end{align}

\paragraph{Inverse Propensity Score (IPS)} This is a recently proposed effective approach to infinite-horizon OPE~\citep{liu2018breaking}, under the naive assumption of no hidden confounding. This method works by fitting both inverse propensity scores and the state density ratio, using similar conditional moment conditions as in \cref{sec:state-density-ratio}.

\paragraph{Black-Box} This is a state-of-the-art approach to OPE~\citet{mousavi20blackbox}, which is similar in nature to \textit{IPS} but works under more general assumptions and tends to be more robust to behavior data sampling distributions than IPS . It also naively assumes no hidden confounding.

\subsection{Hyperparameter Details}
\label{apx:hyperparameters}


\paragraph{Estimating State Density Ratio.} As mentioned in \cref{sec:state-density-ratio}, we let $\mathcal H$ and $\mathcal D$ be norm-bounded RKHSs. Specifically, in both cases we use the identity kernel ($k(s,s') = \indicator{s = s'}$). In each case rather than choosing an explicit radius for the RKHS ball, we apply Lagrangian regularization, using a regularization coefficient of $10^{-8}$ in both cases (note that we describe how to incorporate this regularization for both the interior and exterior optimization problems in \cref{apx:representer-details}). In addition we use $\lambda_c = 10^{-8}$. Furthermore, we initialize $\tilde d$ to be a vector of all ones, and we iterate the min-max calculation of $\hat d$  five times, each time using the previous iterate solution as $\tilde d$.

\paragraph{Calculating Optimal Balancing Weights.} We use the following kernel for our RKHS for $\mathcal G_K$: $k((s,u), (s',u')) = 0.5 \indicator{s = s'} + 0.5 \indicator{u = u'}$, which takes into account the tuple structure of the input of $\mu$. In addition we use $\lambda=10^{-3}$ in all experiments, as we found this gave consistently good performance (as in \citet{bennett2019policy}, we find that small values of $\lambda$ perform well).

\paragraph{IPS and Black-Box.} In general, both of these approaches use neural networks as parametric models to learn the weights of the estimator. However, both environments that we have studied in this paper (i.e., confounded Modelwin and GridWorld) have finite and discrete state space. Therefore, as suggested in Section 5 of \cite{liu2018breaking} (and similarly in \cite{mousavi20blackbox}) we can optimize the weights of the estimator in the space of all possible functions. This corresponds to using a delta kernel in terms of the RKHS used for defining the maximum mean discrepancy in both methods. Accordingly, minimizing loss functions in both baselines (i.e., eq. (12) in \cite{liu2018breaking} and eq. (11) in \cite{mousavi20blackbox}) reduce to quadratic optimization problems, which we solve using constrained optimization by linear approximation (COBYLA).

\subsection{Environment Details}
\label{apx:env}
\paragraph{C-Modelwin.} C-Modelwin has 3 states (denoted $s_0$, $s_1$, and $s_2$) and 2 actions (denoted $a_0$ and $a_1$). The agent always begins in $s_0$. At time $t$, the agent chooses between the actions $a_0$ and $a_1$ with probabilities $1-\pi-U_t$ and $\pi+U_t$ respectively regardless of the current state, where $\pi$ is a scalar policy parameter. In our experiments, we use a behavior policy with $\pi=0.7$, and an evaluation policy with $\pi=0.1$. In addition, $U_{i:n}$ are iid variables taking value $0.1$ or $0.2$ with probabilities $0.3$ and $0.7$ respectively.

Transitions and rewards occur as follows. If the agent is in state $s_0$ at time $t$ and takes action $a_0$, it transitions to $s_1$ or $s_2$ with probabilities $0.7+U_t$ and $0.3-U_t$ respectively. Alternatively, if it takes action $a_1$ in state $s_0$ then it transitions to $s_1$ or $s_2$ with probabilities $0.3+U_i$ and $0.7-U_i$ respectively. In either case it receives zero reward transitioning from $s_0$. If the agent is in state $s_1$ or $s_2$ it transitions to $s_0$, regardless of the action taken. Furthermore, when it transitions from $s_1$ to $s_0$ it receives a reward of $10 + 20 U_i$, and when it transitions from $s_2$ to $s_0$ it receives a reward of $-10 - 20 U_i$. In both cases the reward doesn't depend on the action taken.

\paragraph{GridWorld.} The environment consists of a $10\times 10$ grid, and each state corresponds to the agent's location in the grid (meaning that there are 100 different states). The agent starts from the bottom-left of the grid, and its goal is to reach the top-right of the grid. There are four possible actions: moving \textit{up} ($a_{0}$), \textit{right} ($a_{1}$), \textit{down} ($a_{2}$), and \textit{left} ($a_{3}$). We consider a class of hierarchical policies that first decide whether to move towards the top-right or towards the bottom-left, and then consider whether to move up or right (in case of moving towards top-right), or whether to move down or left (in case of moving towards bottom-left). Specifically, we consider policies that are parameterized by a single scalar parameter $\pi$. At time $t$, the agent first decides to move towards the bottom-left with probability $\pi+U_t$, or the top-right with probability $1-\pi-U_t$. In the case of moving towards the bottom-left, the agent moves down with probability $0.5\pi+U_t$, or left with probability $1-0.5\pi-U_t$. Converseley, in the case of moving towards the top-right, the action taken depends on whether the agent is above or below the diagonal from the bottom-left to top-right: if the agent is below this diagonal they move up with probability $\pi+U_t$ or right with probability $1-\pi-U_t$; if they are above this diagonal they move up with probability $1-\pi-U_t$ or right with probability $\pi+U_t$; and if they are on the diagonal they move up with probability $0.5\pi+0.5U_t$ or right with probability $1 - 0.5\pi-0.5U_t$. As in C-ModelWin, the confounders $U_{1:n}$ are iid variables taking value 0.1 or 0.2 with probabilities 0.3 and 0.7 respectively, and we use $\pi=0.7$ for the behavior policy, and $\pi=0.1$ for the evaluation policy.

State transitions are mostly simple and deterministic; unless the agent is at the goal position of the top-right corner of the grid, it moves one space in the direction indicated by the action (up, right, down, or left). In the case that the agent cannot move in that direction because they are at the edge of the grid (for example if it is at the very right and takes the right action) they simply do not move. On the other hand if the agent is at the top-right corner before taking the action, they transition to the bottom-left corner regardless of the action taken.

Rewards are also simple and deterministic. At time $t$, if the agent is at the goal position of the top-right corner it receives a reward of $100 + 100U_t$, regardless of the action taken. Otherwise, it receives a deterministic reward based on the action taken regardless of the state: $1 + 20U_t$ for up, $1 + 30U_t$ for right, $-1 - 30U_t$ for down, and $-1 - 40U_t$ for left. Note that the agent still receives this reward if it is at the edge of the grid and therefore cannot move.



\subsection{Model Misspecification Details}\label{apx:model-misspecification}

As discussed in the \cref{sec:exp}, in our sensitivity to model misspecification experiments we assume confounders are distributed according to $\alpha \mathcal{P}_{\text{iid}} + (1-\alpha) \mathcal{P}_{\text{alt}}$ where $\mathcal{P}_{\text{iid}}$ denotes the original distribution in which confounder values all independent, $\mathcal{P}_{\text{alt}}$ denotes a distribution in which the confounder value at time $t$ depends on the confounder value at time $t-1$, and $\alpha$ is a model hyperparameter.

Next, as described in \cref{apx:env}, in both environments the original model $\mathcal P_{\text{iid}}$ is given by a simple categorical distribution, where each confounder takes the value 0.1 or 0.2 with probabilities 0.3 and 0.7 respectively.
On the other hand, in the alternative model $\mathcal P_{\text{alt}}$ the confounder still takes the value 0.1 or 0.2, with probabilities that depend on the previous confounder value. Specifically, for the initial time step the respective probabilities are 0.3 and 0.7, as in $\mathcal P_{\text{iid}}$, and for future time steps the respective probabilities are 0.08 and 0.92 if the previous confounder value was 0.1, or 0.82 and 0.18 if the previous confounder value was 0.2.


\subsection{Posterior Noise Injection Details}\label{apx:posterior-noise}

We describe here both how we inject noise in the posterior distributions $\varphi(z)$, and how we measure this noise. Recall that $\varphi(z)$ is shorthand for the posterior distribution of $U$ given $Z=z$, that is $\varphi(z)(u) = P(U=u \mid Z=z)$. In our experiments all $U$ and $Z$ values are discrete, so we have a finite number of posterior distributions $\varphi(z)$, each represented by a finite-length vector. Let $\text{logits}(p)$ denote the vector of log-odds corresponding to the vector of probabilities $p$. Then for each possible value $z$, we independently injected noise in $\varphi(z)$ by adding a random Gaussian vector to $\text{logits}(\varphi(z))$, and then converting the perturbed logits back to probabilities (by taking the expits of the vector entries and re-normalizing). This was done for a wide variety of different variances of the random Gaussian vectors (all with spherical covariances).

It is difficult to interpret the scale of posterior error caused by a given variance for the Gaussian vector we added to the posterior logits, so we came up with the more interpretable metric average standard deviation (ASD). In this metric the average is taken over the distribution of $Z$ values and levels of $U$, and the standard deviation is taken over the distribution of random noise vectors. Formally, let $n_s$ be a number of $Z$ values to sample from the stationary distribution of $\pi_b$, let $n_e$ be a number of random Gaussian vectors to sample for each sampled $Z$ value, and let $n_u$ be the number of levels of $U$. In practice in our experiments we use $n_e=50$ and $n_s=5$. In addition, let $Z_i$ be the $i$'th sampled $Z$ value, let $\epsilon_{i,j}$ be the $j$'th sampled Gaussian vector for the $i$'th sampled $Z$ value. In addition let $\psi(Z_i, \epsilon_{i,j})$ denote the vector of probabilities given by perturbing $\text{logits}(\varphi(Z_i))$ by $\epsilon_{i,j}$, as described above. Then the ASD metric is given by
\begin{equation*}
    ASD = \frac{1}{n_s n_u} \sum_{i=1}^{n_s} \sum_{u=1}^{n_u} \left( \frac{1}{n_e - 1} \sum_{j=1}^{n_e} \left( \psi(Z_i, \epsilon_{i,j})_u - \frac{1}{n_e} \sum_{j'=1}^{n_e} \psi(Z_i, \epsilon_{i,j'})_u \right)^2 \right)^{1/2}
\end{equation*}


\subsection{Additional Plots}\label{apx:additional-plots}

In this section we present sensitivity of the direct method and doubly robust estimator to model misspecification and noise in the oracle for the posterior distribution of confounders. For the sake of visualization and clarity, we have repeated plots of off-policy estimates and RMSEs of different methods.

\begin{figure}[H]
\centering
\subfigure{\includegraphics[width= 0.24\textwidth]{Figs/modelwin_estimate.png}}
\subfigure{\includegraphics[width= 0.24\textwidth]{Figs/modelwin_rmse.png}}
\subfigure{\includegraphics[width= 0.25\textwidth]{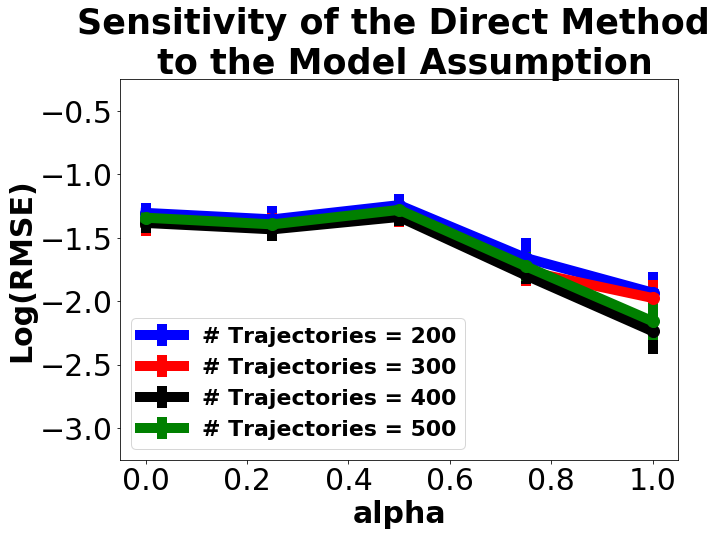}}
\subfigure{\includegraphics[width= 0.25\textwidth]{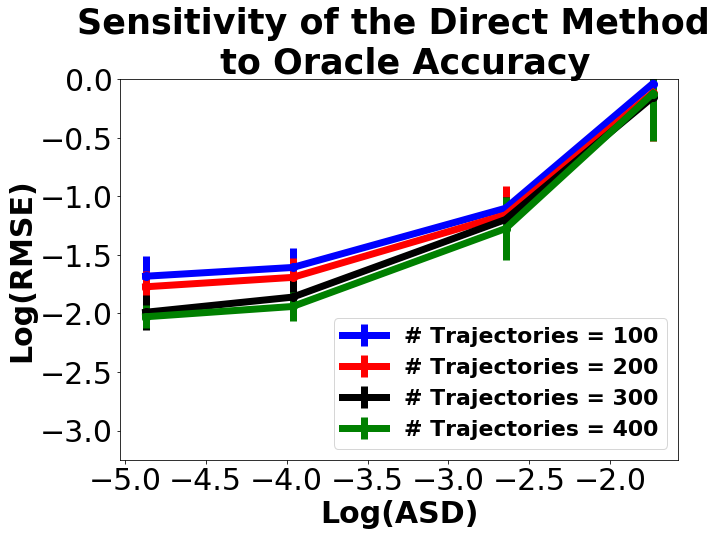}}
\caption{C-ModelWin Results. From left to right: The off-policy estimate, The $\log({\rm RMSE})$ of different methods as we change the number of trajectories, sensitivity of the direct method to model misspecification, and to noise in the confounders posterior distribution.}
\label{fig:modelwin_direct}
\end{figure}

\begin{figure}[H]
\centering
\subfigure{\includegraphics[width= 0.24\textwidth]{Figs/grid_estimate.png}}
\subfigure{\includegraphics[width= 0.24\textwidth]{Figs/grid_rmse.png}}
\subfigure{\includegraphics[width= 0.25\textwidth]{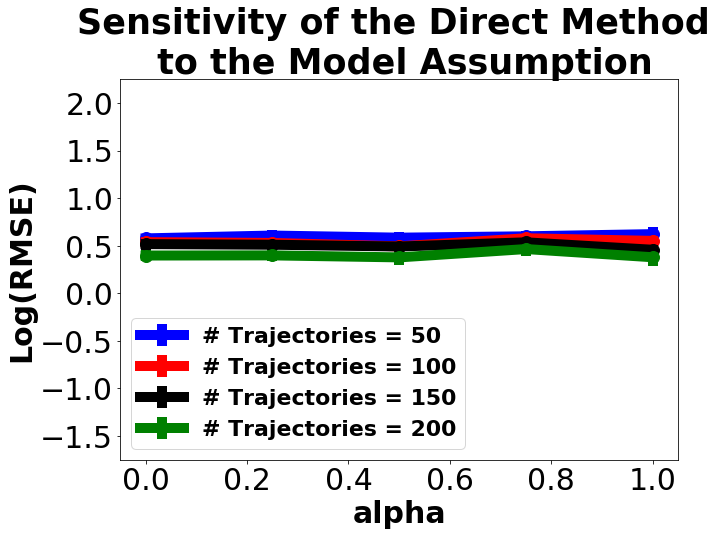}}
\subfigure{\includegraphics[width= 0.25\textwidth]{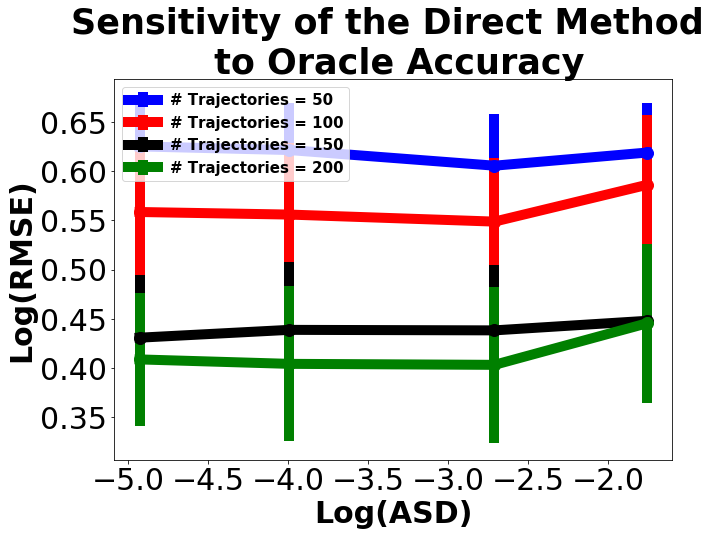}}
\caption{Confounded GridWorld Results. From left to right: The off-policy estimate, The $\log({\rm RMSE})$ of different methods as we change the number of trajectories, sensitivity of the direct method to model misspecification, and to noise in the confounders posterior distribution.}
\label{fig:grid_direct}
\end{figure}

\begin{figure}[H]
\centering
\subfigure{\includegraphics[width= 0.24\textwidth]{Figs/modelwin_estimate.png}}
\subfigure{\includegraphics[width= 0.24\textwidth]{Figs/modelwin_rmse.png}}
\subfigure{\includegraphics[width= 0.25\textwidth]{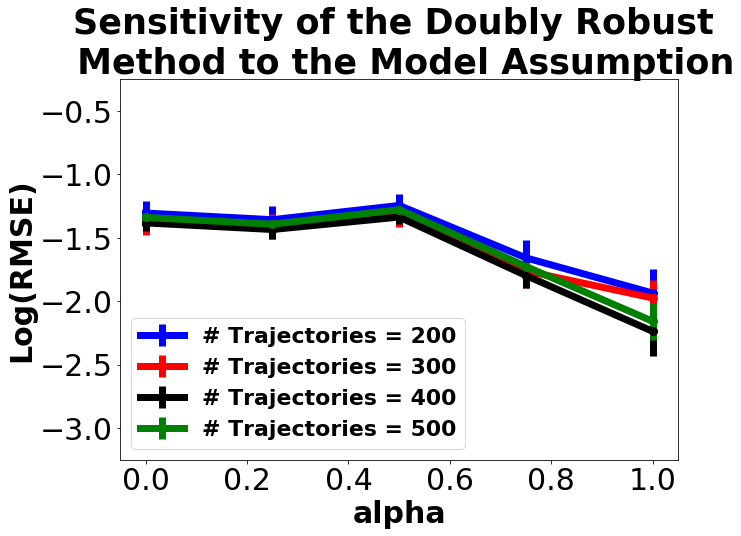}}
\subfigure{\includegraphics[width= 0.25\textwidth]{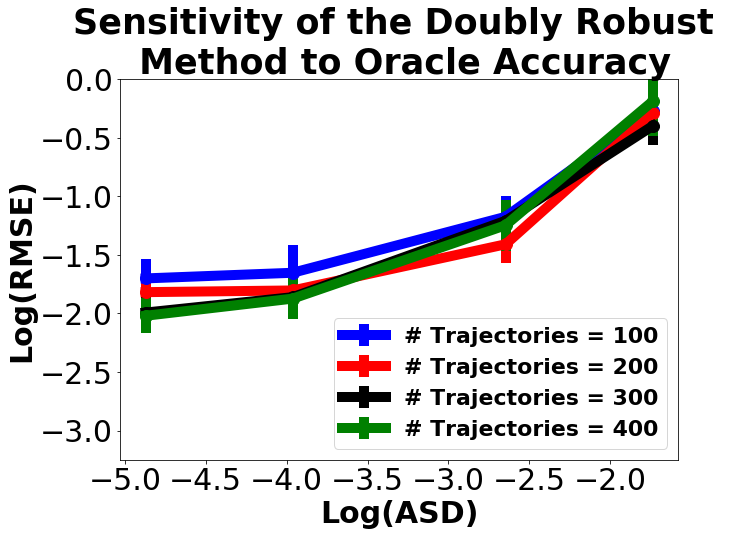}}
\caption{C-ModelWin Results. From left to right: The off-policy estimate, The $\log({\rm RMSE})$ of different methods as we change the number of trajectories, sensitivity of the doubly robust estimator to model misspecification, and to noise in the confounders posterior distribution.}
\label{fig:modelwin_dr}
\end{figure}

\begin{figure}[H]
\centering
\subfigure{\includegraphics[width= 0.24\textwidth]{Figs/grid_estimate.png}}
\subfigure{\includegraphics[width= 0.24\textwidth]{Figs/grid_rmse.png}}
\subfigure{\includegraphics[width= 0.25\textwidth]{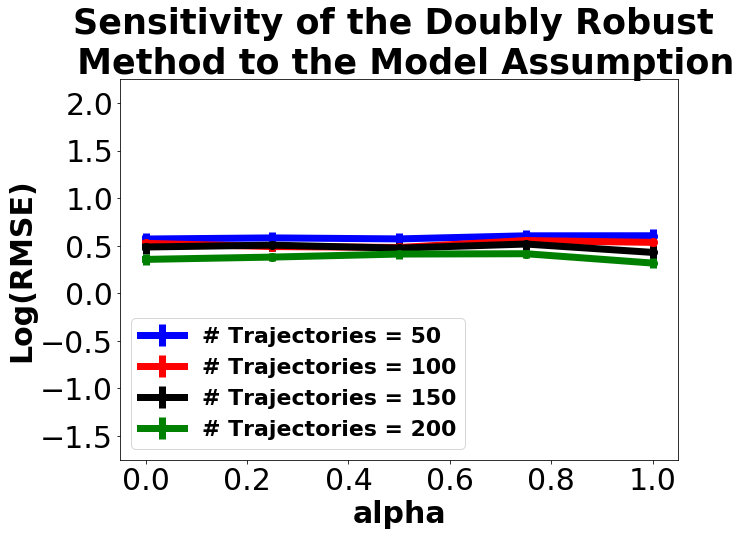}}
\subfigure{\includegraphics[width= 0.25\textwidth]{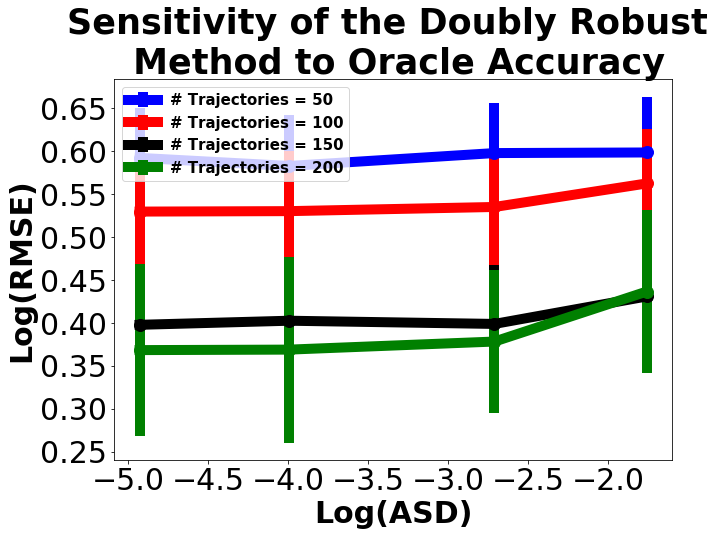}}
\caption{Confounded GridWorld Results. From left to right: The off-policy estimate, The $\log({\rm RMSE})$ of different methods as we change the number of trajectories, sensitivity of the doubly robust estimator to model misspecification, and to noise in the confounders posterior distribution.}
\label{fig:grid_dr}
\end{figure}

\end{document}